\documentclass[foot, reqno]{amsart}
\usepackage{amsaddr}
\usepackage[backend=biber,style=numeric,sorting=nyt, natbib=true]{biblatex}
\usepackage[english]{babel}
\usepackage{csquotes}
\usepackage{amsmath,amssymb,latexsym}
\usepackage{graphicx}
\usepackage{xcolor}
\usepackage{float}
\usepackage{multicol}
\usepackage{subcaption}
\usepackage{enumerate}
\usepackage[shortlabels]{enumitem}
\usepackage{makecell}
\usepackage{booktabs}

\usepackage{xurl}
\usepackage{hyperref}

\newcommand{\mathd}{\mathrm{d}}
\newcommand{\nobracket}{}

\newcommand{\tmop}[1]{\ensuremath{\operatorname{#1}}}

\newcommand{\tmtextbf}[1]{\text{{\bfseries{#1}}}}

\newtheorem{theorem}{Theorem}[section]
\newtheorem{lemma}[theorem]{Lemma}
\newtheorem{proposition}{Proposition}
\theoremstyle{definition}

\theoremstyle{remark}
\newtheorem{remark}[theorem]{Remark}

\begin{document}
\title{Affine Invariant ensemble transform methods to improve predictive uncertainty in neural networks}

\author{Diksha Bhandari$^{1*}$, Jakiw Pidstrigach$^1$, Sebastian Reich$^1$}
\thanks{$^{*}$Corresponding author: Diksha Bhandari}
\email{diksha.bhandari@uni-potsdam.de}
\email{pidstrigach@uni-potsdam.de}
\email{sebastian.reich@uni-potsdam.de}
\address{$^1$University of Potsdam,
Institute of Mathematics, Karl-Liebknecht Str. 24/25, D-14476 Potsdam, Germany}
\thanks{The research has been partially funded by the Deutsche Forschungsgemeinschaft (DFG)-   Project-ID 318763901 - SFB1294. The authors would also like to thank the Isaac Newton Institute for Mathematical Sciences, Cambridge, for support and hospitality during the programme  {\it The Mathematical and Statistical Foundation of Future Data-Driven Engineering} where work on this paper was undertaken. This work was supported by EPSRC grant no EP/R014604/1.}

\subjclass[2020]{62F15, 62J02, 65C05, 68T37, 34F05}


\keywords{Bayesian inference, Logistic regression, Interacting particle systems, Ensemble Kalman filter, Uncertainty quantification, Affine invariance}

\begin{abstract}
We consider the problem of performing Bayesian inference for logistic regression using appropriate extensions of the ensemble Kalman filter. Two interacting particle systems are proposed that sample from an approximate posterior and prove quantitative convergence rates of these interacting particle systems to their mean-field limit as the number of particles tends to infinity. Furthermore, we apply these techniques and examine their effectiveness as methods of Bayesian approximation for quantifying predictive uncertainty in neural networks.
\end{abstract}

\maketitle


\section{Introduction}
The task in inverse problems is the inference of an unknown parameter $\theta \in  \mathbb {R}^{D}$ from noisy observations $d \in \mathbb {R}^N$, which are generated through
\begin{equation*}\label{eqn:inv}
    d = G (\theta) + \eta,
\end{equation*}
where $G$ denotes a forward map from model parameters to observable output data $d$ and $\eta$ denotes observational noise, which is commonly assumed to be Gaussian, that is, $\eta \sim \mathcal N(0, P_\eta)$. 
A Bayesian inverse problem can be formulated as producing samples from a random variable $\theta$ conditioned on $d$. Given the prior $\pi_{\rm prior}(\theta)$ on $\theta$, the inverse problem is formulated as finding the posterior $\pi_{\rm post}(\theta)$ on $\theta$ given $d$. By Bayes theorem, the posterior distribution can be written in terms of the prior density $\pi_{\rm prior}$ and the negative log-likelihood or loss function $\Psi: \mathbb {R}^{D} \rightarrow \mathbb {R}$ as 
\begin{align}\label{eqn:sampling}
	\mathrm{d}\pi_{\rm post}(\theta) \propto \exp(-\Psi(\theta)) \mathrm{d}\pi_{\rm prior}(\theta).
\end{align}
The problem of sampling from a target distribution is fundamental to Bayesian statistics, machine learning, and data assimilation. Considerable research has been done on sampling methods for Bayesian inverse problems, mainly in the case of $l_2$-loss functions \cites{ding2021ensemble, Haber2018NeverLB, Huang_2022, Kovachki_2019}. In this work, we will focus on $\Psi$ being the cross-entropy loss instead, which is used in logistic regression problems. 

We will introduce two methods, closely related to the algorithms introduced in \cite{pidstrigach2022affine} to approximate the posterior distribution for Bayesian logistic regression. We will prove that the methods are well-defined and provide quantitative bounds on their convergence in the large-ensemble (mean-field) limit.

We will then apply these methods to Bayesian neural networks. This gives us the possibility to quantify uncertainty in the model weights and the network outputs. Our numerical experiments show that our methods outperform the state-of-the-art methods for this problem.

\subsection{Classification and logistic regression}\label{sec:logistic_regression}
In this work, we focus on classification problems, that is, the problems arising from classifying objects into classes $C_1, \ldots, C_k$. For notational convenience, we will focus on the case of binary classification, i.e., we only have classes $C_1$ and $C_2$. However, all methods can be generalized to $k$ classes, see \cite[Section 7.2]{pidstrigach2022affine}.

Consider the data set 
\[
\mathcal{D} = \{(\phi^n, d^n)\}_{n = 1}^{N},
\] 
where $d^n \in \{0,1\}$ are targets for the binary classification case with input features $\phi^n \in \mathbb{R}^D \text{ for } n=1,...,N$. The $\phi^n$'s can either be the representation of the data in any numerical form, or stem from a feature map. In Section \ref{sec:llba}, we will train a neural network for the use as feature map.  
We also introduce the shorthand notation
\begin{equation*}
	\Phi = (\phi^1,...,\phi^N) \in \mathbb R ^{\mathit {D \times N}}
\end{equation*}
and define a parametric family of models as follows. 
Given a parameter $\theta$, the probability of an example belonging to class $C_1$ will be given by
\begin{equation}\label{eqn:prob}
\mathbb{P}_\theta[\phi \in C_1] = \sigma(\langle \theta, \phi \rangle),
\end{equation}
where $\langle \cdot, \cdot \rangle$ is the inner product between two vectors and $\sigma$ is the sigmoid function
\begin{equation*}\label{eqn:sigmoid}
	\sigma(z) = \frac{1}{1 + \exp{(-z)}}.
\end{equation*}
The probability of an example belonging to the complimentary class $C_2$ is given by $\mathbb{P}_\theta[\phi \in C_2] = 1-\mathbb{P}_\theta[\phi \in C_1]$.
The negative log-likelihood of the given dataset $\mathcal{D}$ under our parametric model is given by the cross-entropy loss function
\begin{equation}\label{eqn:cee}
	\Psi(\theta) = -\sum_{n=1}^N \{d^n\log(y_n(\theta)) + (1 - d^n)\log(1 - y_n (\theta))\}, 
\end{equation}
where
\begin{equation}\label{eqn:y}
	y_n(\theta) = \sigma(\langle \theta, \phi^n\rangle) = \mathbb{P}_\theta[\phi^n \in C_1].
\end{equation}
For $n =1, \ldots, N$, we introduce the vector $y (\theta) \in \mathbb R^{N}$ as 
\begin{equation*}
    y(\theta) = (y_1(\theta), \ldots, y_{n}(\theta))^{\rm T},
\end{equation*}
and the vector of target data labels $d \in \mathbb R^{N}$ as
\begin{equation*}
    d = (d^1, \ldots, d^{n})^{\rm T}.
\end{equation*}

We investigate sampling methods based on the ensemble Kalman filter (EnKF) \cite{10.5555/1206873} and its extension to Bayesian logistic regression as already considered in \cite{pidstrigach2022affine}. As discussed in the previous section, the Bayesian inverse problem now consists of sampling from $\pi_\text{post}$ given by \eqref{eqn:sampling}.

Note that the likelihood in this section stems from linear logistic regression, since $y_n = \sigma(\langle \theta, \phi^n \rangle)$. However, the methods can be generalized to nonlinear logistic regression, i.e., $y_n = \sigma(f^n(\theta))$, with $f^n(\theta)$ being nonlinear, see \cite[Section 7.2]{pidstrigach2022affine}. Furthermore, the method proposed in Section \ref{sec:infinite_det_som} can be implemented in a derivative-free manner, as discussed in \cite[Section 7.1]{pidstrigach2022affine}.

\subsection{Literature review}
Since its introduction in \cite{https://doi.org/10.1029/94JC00572}, the EnKF has been a popular choice for performing data assimilation tasks due to its robustness and wide applicability. In particular, the EnKF has shown promise as a derivative-free Bayesian inference technique \cites{10.5555/1206873, pidstrigach2022affine, Huang_2022}. More recently, the EnKF has been combined with sampling methods for Bayesian inference \cites{ding2021ensemble, ding2021ensemblesampler} in order to transform random samples at $s =0$ into samples from the posterior as $s \rightarrow \infty$. 
Studying the theoretical aspects of homotopy methods based on the EnKF has also been an active area of research \cites{schillings2017analysis, schillings2018convergence, pidstrigach2022affine, Reich2011ADS, articleReich}. Furthermore, it should also be noted that both the EnKF and ensemble Kalman inversion (EKI) \cite{Kovachki_2019} can be cast within an interactive particle setting. However, most of the work is done in the case of a quadratic likelihood $\Psi(\theta)$ or a perturbation of a quadratic likelihood. The work \cites{Kovachki_2019,pidstrigach2022affine} introduces multiple EnKF-type methods to study Bayesian logistic regression, i.e., $\Psi$ being the negative log-likelihood of a logistic regression model. In this paper, we further develop two of the methods proposed in \cite{pidstrigach2022affine} by taking inspiration from \cites{doi:10.1137/19M1304891, Huang_2022, e23080990,doi:10.1137/19M1303162} and deploy them for uncertainty quantification in Bayesian logistic regression.

As neural networks have become widely used in critical applications, the need for accurate uncertainty quantification has grown \cite{Gal2016UncertaintyID}. To address this, a popular technique is using Bayesian inference \cite{reich_cotter_2015}, which allows machine learning algorithms to assign low confidence to test points that deviate significantly from the training data or prior information. Bayesian neural networks are commonly used for this purpose \cites{Wilson2020TheCF, ABDAR2021243,MacKay1992TheEF, 10.5555/3045118.3045290, pmlr-v80-zhang18l, pmlr-v48-gal16, NEURIPS2019_8558cb40, MacKay1991APB}. Popular approximate Bayesian approaches include traditional variational inference \cites{10.5555/3045118.3045290,NIPS2011_7eb3c8be}, Laplace approximation \cite{10.1162/neco.1992.4.3.448}, Sampling-based approaches like Hamiltonian Monte Carlo (HMC) \cite{neal2012bayesian} and stochastic gradient Langevin dynamics \cite{Welling2011BayesianLV}. More recently, in an attempt to make these Bayesian approaches more computationally feasible, many methods of partial stochastic networks have been considered \cites{sharma2023bayesian, Kristiadi2020BeingBE, Daxbergeretal21, pmlr-v37-snoek15}. Moreover, many non-Bayesian methods \cites{Lakshminarayanan2016SimpleAS, Hein_2019_CVPR, liang2018enhancing} for uncertainty quantification have also been proposed. These methods aim to provide usable estimates of predictive uncertainty even in the presence of limited data. 
 
\subsection{Outline and our contribution}
The fundamental idea underlying this study is to develop efficient Bayesian inference methods for logistic regression based on the EnKF. Furthermore, we apply these methods to uncertainty quantification in neural networks and compare them to the state-of-the-art. To that end we will derive two interacting particle systems (IPSs) which sample from an approximate posterior. Moreover, we prove quantitative convergence rates of the IPSs to their mean-field-limit as the number of particles tends to infinity. We demonstrate the efficacy of the proposed methods for estimating uncertainty in neural networks through numerical experiments for binary classification in Section \ref{sec:results}.

The remainder of this paper has been organized as follows. The dynamical system formulations for ensemble transform methods based on the homotopy approach and the infinite time second-order dynamical sampler for Bayesian inference in logistic regression are described in Section \ref{sec:methods}. Section \ref{sec:mfl} analyzes the proposed ensemble transform  methods and derives their mean-field limits. Therein, we quantitatively bound the distance of the $J$-particle system to its infinite-particle limit in Theorem \ref{thm:mfl_deterministic_aldi}. Efficient methods of time-stepping of interacting particle systems for robust implementations and pseudocode describing the algorithms that we employ in this paper are discussed in Section \ref{sec:time_stepping}. Section \ref{sec:llba} introduces the framework for application of proposed methods to Bayesian logistic regression for uncertainty quantification in neural networks. Experimental results on predictive uncertainty of neural networks using proposed Bayesian methods for inference are shown in Section \ref{sec:results}. These experiments demonstrate how reliable the uncertainty estimates are for different methods considered.

%
%

\section{Dynamical system formulations}\label{sec:methods}
We denote by $\theta^1_s, \ldots \theta^J_s$ an ensemble of $J$ particles, indexed by a time parameter $s \geq 0$. We denote by $m_{\theta_s}$ the empirical mean of the ensemble,
\begin{equation}\label{eqn:emp_mean}
	m_{\theta_s} = \frac{1}{J} \sum_{j=1}^J \theta_{s}^{j},
\end{equation}
and by $P_{\theta_s}$ the empirical covariance matrix,
\begin{equation}\label{eqn:emp_cov}
	P_ {\theta_s} = \frac{1}{J} \sum_{j=1}^{J}(\theta_{s}^{j}- m_{s} ) (\theta_ {s}^{j}- m_{s})^{\rm T}.
\end{equation} 
We introduce the matrix of ensemble deviations $\Theta_s \in \mathbb R^{D \times J}$ as
\begin{equation}\label{eqn:ens_dev}
    \Theta_s = \left(\theta_{s}^{1}-m_{\theta_s},\theta_{s}^{2}-m_{\theta_s},\ldots,\theta_{s}^{J}-m_{\theta_s}  \right).
\end{equation}
We will adapt the convention that upper indices stand for the ensemble index, while lower indices stand for the time index or the $i$th entry of an $N$-dimensional vector.
Associated to the particle system is the empirical measure
\begin{equation*}
	\mu_{\theta_s} = \frac{1}{J} \sum_{j=1}^J \delta_{\theta^j_s},
\end{equation*}
where $\delta_\theta$ stands for the Dirac delta at $\theta$. We denote the expectation with respect to this measure by $\mu_{\theta_s}[f]$, i.e.~for a function $f$,
\[
	\mu_{\theta_s}[f] = \frac{1}{J} \sum_{j=1}^J f(\theta^j_s).
\]
Finally, we assume that the prior measure $\pi_\text{prior}$ is a Gaussian and given as
\[
	\pi_{\rm prior} = \mathcal{N}(m_{\rm prior}, P_{\rm prior}).
\]
In each of the following two subsections, we introduce an IPS to sample from the posterior in a logistic regression problem. 

\subsection{Homotopy using moment matching}\label{sec:homotopy_som}
In homotopy approaches to Bayesian inference, we assume that the initial ensemble $\{\theta^j_0\}$ is distributed according to the prior $\pi_\text{prior}$. One then evolves the ensemble such that at some fixed terminal time $s$, often $s = 1$, the ensemble is distributed according to the posterior $\pi_\text{post}$.
To derive such an evolution, one starts by defining a path $\pi_s$ in the space of measures, which starts at a prior distribution $\pi_0 = \pi_{\rm prior}$ and ends at $\pi_1 = \pi_{\rm post}$ \cites{10.1117/12.839590, Reich2011ADS}.

We will study a homotopy approach introduced in \cite{pidstrigach2022affine}. In this section, we will shortly summarize the resulting differential equations. See \cite[Section 7.1]{pidstrigach2022affine} for more details and the derivation of the equations.

We will need the gradient and Hessian of the negative log-likelihood $\Psi$, defined in \eqref{eqn:cee}. The gradient is given by
\begin{equation}
    \nabla_\theta \Psi(\theta)= \sum_{n=1}^N (y_n(\theta) - d^n)\phi^n = \Phi(y(\theta) - d),
    \label{equ:grad_Psi}
\end{equation}
while the Hessian is
\begin{equation}
    D_\theta^2 \Psi(\theta) = \Phi R(\theta) \Phi^{\rm T}.
    \label{equ:hessian_Psi}
\end{equation}
Here $R(\theta) \in \mathbb{R}^{N \times N}$ is a diagonal matrix with diagonal entries 
\begin{equation*}
    r_{nn} = y_n(\theta)(1 - y_n(\theta)).
    \label{equ:def_R}
\end{equation*}
For $s \in [0,1]$, the dynamical system to transform the prior is given by
\begin{equation}\label{equ:homotopy_ode}
\begin{aligned}
    \frac{\mathd}{\mathd s} \theta_s^j  
    &= -\frac{1}{2} P_{\theta_s} \left(
        \mu_{\theta_s}[D^2_\theta \Psi(\theta)](\theta_s^j - m_{\theta_s})
        + 2 \mu_{\theta_s}[\nabla_\theta \Psi(\theta)]\right)\\
    &=  - \frac{1}{2} P_{\theta_s} \Phi
		\Bigl(\mu_{\theta_s} [R] \Phi^{\rm T} (\theta_s^j - m_{\theta_s}) + 2(\mu_{\theta_s}[y] - d)\Bigr) 
\end{aligned}
\end{equation}
with random initial conditions $ \theta^j_0 \sim \mathcal{N}(m_{\rm prior}, P_{\rm prior})$ i.i.d.~for $j = 1, \ldots, J$.

One can also evolve the mean $m_{\theta_s}$ and ensemble deviations $\Theta_s$ instead of the ensemble $\{\theta_s^j\}_{j=1}^J$. Their evolution equations are given by
\begin{equation}\label{eqn:hom_mean_ode}
    \frac{\mathd}{\mathd s} m_{\theta_s}  =  - P_{\theta_s} \Phi \Bigl( \mu_{\theta_s}[y] - d \Bigr), 
\end{equation}
and
\begin{equation}\label{eqn:hom_dev_ode}
\frac{\mathd}{\mathd s} \Theta_{s}  =  - \frac{1}{2} P_{\theta_s} \Phi \mu_{\theta_s} [R] \Phi^{\rm T} \Theta_{s}
\end{equation}
respectively, where $P_{\theta_s}$ is the empirical covariance matrix given in \eqref{eqn:emp_cov} and $\mu_{\theta_s}[R]$ is defined as
$$\mu_{\theta_s}[R] = \frac{1}{J} \sum_{j=1}^J R(\theta_s^j).$$

\subsection{Deterministic second-order dynamical sampler}\label{sec:infinite_det_som}
Alternatively to transporting from the prior to the posterior in a fixed time interval, one can also construct systems that sample the target distribution as $s \to \infty$. Markov Chain Monte Carlo algorithms are the most famous family of algorithms with this property. However, they normally work on a single sample trajectory $\theta_s$ instead of an ensemble.

The algorithm introduced in \cite[Section 5]{pidstrigach2022affine} combines the homotopy approaches with overdamped Langevin dynamics to motivate an IPS that approximates the posterior as $s \to \infty$. The system of equations is given by
\begin{equation}
\begin{split}
\frac{\mathd \theta_s^j}{\mathd s}  & =  -\frac{1}{2}P_{\theta_s} \Biggl( \Phi \Bigl( \mu_{\theta_s}[R] \Phi^{\rm T}(\theta_s^j -                                                     m_{\theta_s}) + 2 (\mu_{\theta_s}[y] - d) \Bigr) \\
                                         & \quad + P_{\rm{prior}}^{-1}(\theta^j_s + m_{\theta_s} - 2 m_{\rm{prior}}) \Biggr)
                                         + P_{\theta_s}^{1/2} \mathd W_s^j,
\end{split}
\label{equ:stochastic_second_order}
\end{equation}
where $W_s^j$ denotes $D$-dimensional standard Brownian motion.
For details on the derivation see \cite[Section 5]{pidstrigach2022affine}. Note, that similar systems are also introduced in \cites{doi:10.1137/19M1304891, Huang_2022, garbuno2020interacting}.

We now modify \eqref{equ:stochastic_second_order} by replacing the stochastic driving term $P_{\theta_s}^{1/2} \mathd W_s^j$ by $\frac{1}{2}(\theta^j_s - m_{\theta_s}) \mathd s$, rendering the system deterministic except for the choice of the initial ensemble $\{\theta_0^j\}_{j=1}^J$. The advantage of making it deterministic is that we can perfectly assess convergence of the algorithm, since the particles will stop moving when they reach equilibrium.

Therefore, the IPS ODE we will study is given by
\begin{equation}\label{equ:aldiode}
\begin{split}
\frac{\mathd \theta_s^j}{\mathd s}  & =  -\frac{1}{2}P_{\theta_s} \Biggl( \Phi \Bigl( \mu_{\theta_s}[R] \Phi^{\rm T}(\theta_s^j -                                                     m_{\theta_s}) + 2 (\mu_{\theta_s}[y] - d) \Bigr) \\
                                         & \quad + P_{\rm{prior}}^{-1}(\theta^j_s + m_{\theta_s} - 2 m_{\rm{prior}}) \Biggr)
                                         + \frac{1}{2}(\theta^j_s -m_{\theta_s}) 
\end{split}
\end{equation}
with random initial conditions $ \theta^j_0 \sim \mathcal{N}(m_{\rm 0}, P_{\rm 0})$ i.i.d. for $j = 1, \ldots, J$. 

As already utilized in Section \ref{sec:homotopy_som}, to propagate the ensemble $\{\theta^j_{s}\}_{j=1}^J$ in the interacting particle system \eqref{equ:aldiode}, we can equivalently evolve the mean $m_{\theta_s}$ and ensemble deviations $\Theta_s$ by
\begin{equation}\label{eqn:aldi_mean_ode}
    \frac{\mathd}{\mathd s} m_{\theta_s}  =   - P_{\theta_s}\left( \Phi ( \mu_{\theta_s}[y] - d) + P_{\rm{prior}}^{-1} (m_{\theta_s} -  m_{\rm{prior}})\right) , 
\end{equation}
and
\begin{equation}\label{eqn:aldi_dev_ode}
\frac{\mathd}{\mathd s} \Theta_s  =  - \frac{1}{2} P_{\theta_s} \left(\Phi \mu_{\theta_s} [R] \Phi^{\rm T} \Theta_s + P_{\rm{prior}}^{-1} \Theta_s  \right) + \frac{1}{2} \Theta_s
\end{equation}
respectively, where $P_{\theta_s}$ is the empirical covariance.

Our theory section will focus on the deterministic second-order sampler. To motivate this, we now highlight some attractive properties of that algorithm. We will later see that it also outperforms other methods numerically.

\begin{remark}
It is important to note that in contrast to the stochastic system \eqref{equ:stochastic_second_order}, the deterministic system \eqref{equ:aldiode} will stop evolving when it reaches the equilibrium. This is an attractive property for numerical approximation, since we know when we can stop simulating the IPS. Moreover, if the expectations $\mu_{\theta_s}[y]$, $\mu_{\theta_s}[R]$ of $y(\theta)$ and $R(\theta)$ could be exactly approximated, then the evolution equations \eqref{eqn:aldi_mean_ode}-\eqref{eqn:aldi_dev_ode} given by the deterministic sampler would give an exact representation for the evolution of the mean $m_{\theta_s}$ and covariance $P_{\theta_s}$ in the limit $J \rightarrow \infty$. However, the stochastic system \eqref{equ:stochastic_second_order} lacks this property.
\end{remark}

\begin{remark}
Note that the homotopy method discussed in \ref{sec:homotopy_som}, the equations will reach their target distribution at a fixed time $s = 1$. However, this requires selecting a small step-size and sampling initial conditions from the right prior. In contrast, the system introduced in this section does not have these requirements.
\end{remark}

The replacement of $P_{\theta_s}^{1/2}$ by $\frac{1}{2}(\theta_s - m_s)$ in \eqref{equ:aldiode}is motivated by the fact, that in the mean-field case, for Gaussian densities, they have the same distributional effect, which we prove in the next section in Proposition \ref{prop:motivation_deterministic_noise}. Furthermore, we will prove that the Gaussian assumption is well-founded in Section \ref{sec:mfl}.



%
%

\section{Theoretical results on mean-field limits}\label{sec:mfl}

In this section, we study what happens to the equations
\eqref{equ:homotopy_ode} and \eqref{equ:aldiode} when we let the number of particles go to infinity. As we will see, the interaction between the particles decreases, and in the limit $J \rightarrow \infty$ all particles follow a deterministic mean-field ODE, independently of the other particles.

To that end, we introduce some notation. Both of the IPS in Section \ref{sec:methods} only depend on the other
particles through their empirical measure. Therefore, we can rewrite
\eqref{equ:homotopy_ode} and \eqref{equ:aldiode} as
\begin{equation}
    \mathd \theta_s^j = b (\mu_{\theta_s}) (\theta^j),
    \label{equ:ips_abstract_with_b}
\end{equation}
where $b (\mu_{\theta})$ is defined implicitly by \eqref{equ:homotopy_ode} and \eqref{equ:aldiode}. By the continuity equation, we know that if the
particles are evolved using the drift \eqref{equ:ips_abstract_with_b}, then their empirical measure is a weak solution to the partial differential equation (PDE)
\begin{equation*}
  \partial_s \mu_{\theta_s} (\theta) = - \tmop{div} (\mu_{\theta_s} b
  (\mu_{\theta_s}) (\theta)), \qquad \mu_{\theta_0} = \frac{1}{J} \sum_{j =
  1}^J \delta_{\theta_i},
\end{equation*}
where $\delta_{\theta}$ is the Dirac-delta distribution at $\theta$. Due to
the dependence of $b$ on $\mu_{\theta_s}$, this PDE is typically nonlinear.
However, since $b$ only depends on the other particles through the empirical
measure, the above PDE forms a closed system. Therefore, we abstract the PDE
\begin{equation}
  \partial_s \mu_s (\theta) = - \tmop{div} (\mu_s  b (\mu_s) (\theta))
  \label{equ:meanfieldpde}
\end{equation}
which, given an initial condition $\mu_0$, can be solved at the level of
measures or densities directly. Given such a solution $(\mu_s)_s$ for a fixed initial $\mu_0$, we plug it into \eqref{equ:ips_abstract_with_b}
and get the differential equation
\[ \frac{\mathd}{\mathd t} \eta_s = b (\mu_s) (\eta_s) . \]
Note that  this equation does not constitute an IPS anymore but a mean-field ODE instead; i.e., the particle evolution depends on its own distribution $\mu_s$, which we obtain as the solution to \eqref{equ:meanfieldpde}. 
Furthermore, we find from (\ref{equ:meanfieldpde}) that $\eta_0 \sim \mu_0$ implies $\eta_s \sim \mu_s$ for all $s>0$.

The proofs in the following subsections work now as follows. We fix an
initial condition
\[ \mu_0 =\mathcal{N} (m_0, P_0), \]
for which we obtain the solution $\mu_s$ to \eqref{equ:meanfieldpde}. Then, we
define an intermediate mean-field particle system: 
\begin{equation} 
  \frac{\mathd}{\mathd t} \eta_s^j  = b (\mu_s) (\eta_s^j), \quad \eta_0^j \sim \mu_0 .  \label{equ:meanfieldode}
\end{equation}
Note that the $\eta_s^j$, $j=1,\ldots,J$, are all independent.
We then couple the IPS \eqref{equ:homotopy_ode} or \eqref{equ:aldiode} to \eqref{equ:meanfieldode} by choosing the same initial conditions, i.e., $\theta^j_0 = \eta^j_0$ for all $j = 1, \ldots, J$. We then prove that the expected Wasserstein-distance of the empirical measures 
\begin{equation}
    \mathbb{E}[\mathcal{W} (\mu_{\theta_s}, \mu_{\eta_s})] \to 0
    \label{equ:expected_wasserstein}
\end{equation}
converges to $0$ as $J \rightarrow \infty$. Here, the squared $2$-Wasserstein distance is defined as 
\begin{equation}
    \mathcal{W}(\mu, \nu)^2 = \inf_{\gamma \in \Gamma} \int \| x - y\|^2 \mathd \gamma(x, y), \label{eq:wasserstein_definition}
\end{equation}
where $\Gamma$ is the set of all couplings of $\mu$ and $\nu$, i.e. all measures on $\mathbb{R}^D \times \mathbb{R}^D$ with marginals $\mu$ and $\nu$.
Let us briefly discuss the expectation value in \eqref{equ:expected_wasserstein}. Note that although \eqref{equ:ips_abstract_with_b} and \eqref{equ:meanfieldode} are deterministic, $\mu_{\theta_s}$ as well as $\mu_{\eta_s}$ are random probability measures. However, the randomness comes solely from the initial conditions. Therefore, the expectation in \eqref{equ:expected_wasserstein} is with respect to i.i.d.~initial conditions $\eta^j_0 = \theta^j_0 \sim \pi_0$.

We not only prove \eqref{equ:expected_wasserstein} but are able to obtain a quantitative convergence rate. Since $\mu_{\eta_s}$ consists of point masses at independent samples from $\mu_s$, this shows that for large $J$ the $\theta_s^j$ approximate independent samples from $\mu_s$. One can make the last statement precise by using known rates at which $\mu_{\eta_s}$ converges to $\mu_s$, see \cite{fournier2015rate}.

\medskip

\subsection{Analysing the mean-field systems}

As already discussed, we will prove that the empirical measure $\mu_{\theta_s}$ approximates the solution of the mean-field PDE $\mu_s$
for large ensemble sizes. Therefore, it is instructive to briefly
study $\mu_s$. Since $\mu_0$ is Gaussian in our case and the drift terms $b(\mu_{\theta_s})(\theta)$ are linear in $\theta$, $\mu_s$ will remain Gaussian for all times $s \ge 0$. We denote its mean and covariance by $m_s$ and $P_s$:
\[
    \mu_s = \mathcal{N}(m_s, P_s).
\]
Therefore, solving the mean-field PDE \eqref{equ:meanfieldpde} corresponds to solving an ODE for the mean and covariance of the Gaussian distribution. We will next work out these mean-field ODEs for our two IPS.

\subsubsection{Homotopy using moment matching}
The mean-field limit PDE, corresponding to \eqref{equ:meanfieldpde}, is given by
\begin{align*}
  \partial_s \mu_s (\eta) & =  \frac{1}{2} \tmop{div} \left( \mu_s \left( P_s
  (\Phi \mu_s[R] \Phi^{\rm T}(\eta - m_s) + 2 \Phi (\mu_s[y] - d) \right)
  \right).
\end{align*}
In the case of Gaussian initial conditions, the above PDE is equivalent to solving the following ODEs for the mean and covariance:
\begin{align*}
    \frac{\mathd}{\mathd s} m_s 
    & =  - P_s \Phi (\mu_s[y] - d) = - P_s \mu_s[\nabla_\theta \Psi(\theta)],\\
    \frac{\mathd}{\mathd s} P_s & =  - P_s \Phi \mu_s[R] \Phi^{\rm T}P_s = -P_s \mu_s[D_\theta^2 \Psi(\theta)] P_s,
\end{align*}
where $\nabla_\theta \Psi$ and $D_\theta^2 \Psi$ are the gradient and Hessian of the negative log-likelihood function as defined in \eqref{equ:grad_Psi} and \eqref{equ:hessian_Psi}, respectively. 

\subsubsection{Deterministic second-order dynamical sampler}
In this case, the mean-field limit PDE is given by
\begin{equation}
\begin{aligned}
  \partial_s \mu_s (\eta) & = \frac{1}{2} \tmop{div} \left( \mu_s \left( P_s
  (\Phi \mu_s[R] \Phi^{\rm T}(\eta - m_s) + 2 \Phi (\mu_s[y] - d) \right)
  \right) \\
  & \quad + \frac{1}{2} \tmop{div} \left( \mu_s \left( P_{\tmop{prior}}^{- 1}
  (\eta + m_s - 2 m_{\tmop{prior}})) +  (\eta - m_s) \right)
  \right),
\end{aligned}
\label{equ:second_order_mf_pde}
\end{equation}
where $\mu_0 = \mathcal{N}(m_\text{0}, P_\text{0})$. 
As in the previous subsection, we again derive the mean-field ODEs for the mean and covariance:
\begin{equation}
  \begin{array}{lll}
    \frac{\mathd}{\mathd s} m_s 
     =  - P_s \Phi (\mu_s[y] - d) - P_s
    P_{\tmop{prior}}^{- 1} (m_s - m_{\tmop{prior}}),\\
    \frac{\mathd}{\mathd s} P_s  =  - P_s \Phi \mu_s[R] \Phi^{\rm T}P_s -
    P_s P_{\tmop{prior}}^{- 1} P_s + P_s .
  \end{array} \label{equ:mflmeancovevolution}
\end{equation}
The mean-field ODE  \eqref{equ:meanfieldode} becomes
\begin{equation}
\begin{split}
    \frac{\mathd}{\mathd s} \eta^j_s & =  - \frac{1}{2} P_s \left(\Phi \nobracket
    (\mu_s[R] \Phi^{\rm T}(\eta^j - m_s)) + 2 \Phi (\mu_s[y] - d)\right)\\
                                         & \quad +  \frac{1}{2} P_s P_{\tmop{prior}}^{- 1}
    (\eta^j + m_s - 2 m_{\tmop{prior}})) + \frac{1}{2} (\eta^j - m_s),
  \end{split} \label{equ:aldi mf particles}
\end{equation}
where $\eta_0^j = \theta^j_0 \sim \mathcal{N}(m_0, P_0)$. When the system \eqref{equ:mflmeancovevolution} stops evolving, we have reached the equilibrium distribution 
\[
    \mu_* = \mathcal{N}(m_*, P_*).
\]
To derive equations for $m_*$ and $P_*$, we set the right-hand sides of \eqref{equ:mflmeancovevolution} to zero and obtain:
\begin{align}
    m_* &= m_\text{prior} - P_{\text{prior}}\Phi(\mu_{*}[y] - d) = m_\text{prior} - P_\text{prior}\mu_*[\nabla \Psi(\theta)], \label{equ:second_order_mfl_equ_mean}\\
    P_* &= (\Phi \mu_*[R] \Phi^{\rm T}+ P_\text{prior}^{-1})^{-1} =  (\mu_*[D^2_\theta \Psi(\theta)] + P_\text{prior}^{-1})^{-1}\label{equ:second_order_mfl_equ_cov}.
\end{align}
These are implicit equations in $m_\ast$ and $P_\ast$ and the evolution equations (\ref{equ:second_order_mfl_equ_cov}) can be seen as means to find 
$m_\ast$ and $P_\ast$.
Therefore, in the many-particle and large time limit, we are approximating a Gaussian with mean and covariance given by \eqref{equ:second_order_mfl_equ_mean} and \eqref{equ:second_order_mfl_equ_cov} respectively. 

Approximating a distribution by a Gaussian is also an important topic in variational inference \cite{galy2021flexible}. However, in contrast to popular methods for Gaussian variational inference (see the discussion in \cite{galy2021flexible}), which are based on taking gradients with respect to the mean and covariance of the approximating Gaussian, we do not need to invert the state space ($D \times D$) covariance matrix for our method to work. 

\begin{remark}
In this discussion we want to motivate the replacement of $P_{\theta_s}^{1/2}\mathd W_s^j$ by $\frac{1}{2} (\theta_s^j - m_{\theta_s})$ in Section \ref{sec:infinite_det_som}.
Assume that we have two mean-field systems
\begin{equation}
    \mathd \eta_s = b(\mu_s)(\eta_s) \mathd s + P_s^{1/2} \mathd W_s
    \label{equ:stochastic_noise}
\end{equation}
and
\begin{equation}
    \frac{\mathd}{\mathd s} \eta_s = b(\mu_s)(\eta_s) + \frac{1}{2}(\eta_s - m_s).
    \label{equ:deterministic_noise}
\end{equation}
In both systems, $\mu_s$ denotes the distribution of $\eta_s$, and hence the evolution of $\eta_s$ depends not only on $\eta_s$ but also on its own distribution.
We denote by $m_s$ and $P_s$ the mean and covariance of $\mu_s$. Assuming that $\eta_0 \sim \mu_0$, one can derive evolution equations for $\mu_s$: 
For \eqref{equ:stochastic_noise}, by the Fokker-Planck equation, we get that
\begin{equation}
    \partial_s \mu_s = -\text{div}(\mu_s b(\mu_s)) + \frac{1}{2}\text{div}(P_s \nabla \mu_s).
    \label{equ:stochastic_noise_pde}
\end{equation}
For \eqref{equ:deterministic_noise}, by the continuity equation, we get that
\begin{equation}
    \partial_s \mu_s = -\text{div}(\mu_s b(\mu_s)) - \text{div}(\mu_s (\cdot - m_s)),
        \label{equ:deterministic_noise_pde}
\end{equation}


The replacement is motivated by the fact that in this case (for Gaussian $\mu_s$), the time evolutions on the density level coincide, which we will prove in the following Proposition \ref{prop:motivation_deterministic_noise}.
\end{remark}

\begin{proposition}
    Assume that $\mu_s = \mathcal{N}(m_s, P_s)$ is Gaussian. Then, the right-hand sides of \eqref{equ:stochastic_noise_pde} and \eqref{equ:deterministic_noise_pde} coincide. 
    \label{prop:motivation_deterministic_noise}
\end{proposition}
\begin{proof}
  The $\text{div}(\mu_s b(\mu_s))$ obviously coincide. Furthermore, we get that 
  \[
    \nabla \mu_s = \mu_s \nabla \log \mu_s = -\mu_s P_{s}^{-1}(\cdot - m_s)
  \]
  and therefore
  \[
    \frac{1}{2} \text{div}(P_{s} \nabla \mu_s) = -\frac{1}{2}\text{div}(\mu_s (\cdot - m_s)),
  \]
  and the two equations are the same.
\end{proof}

Therefore, even though the two systems \eqref{equ:stochastic_noise} and \eqref{equ:deterministic_noise} behave differently on a particle level, they will have the same effect on the distributional level in our case, in the large ensemble limit. However, since our goal is to use the empirical measure of the ensemble as an approximation to the posterior, we are only interested in the distribution of the particles, not in the trajectories $(\theta^j_s)_s$ of a single particle. 

\begin{remark}
The work \cite{galy2021flexible} also proposes a deterministic IPS for Gaussian variational inference. While their evolution equations differ from ours, the equilibrium state agrees with (\ref{equ:second_order_mfl_equ_mean})-(\ref{equ:second_order_mfl_equ_cov}).
Furthermore, in contrast to our formulation, the IPS proposed in \cite{galy2021flexible} is not affine-invariant. See \cite{pidstrigach2022affine} for a discussion of affine invariance and \cite{chen2023sampling} for more precise discussion on the formulation proposed in \cite{galy2021flexible}.
\end{remark}

\subsection{Statement of results}

Since the IPS \eqref{equ:homotopy_ode} and \eqref{equ:aldiode} are
similar, differing only in additional terms for \eqref{equ:aldiode}, the
proofs of the following results are quite similar too. Due to \eqref{equ:aldiode} having more terms, in particular also terms that increase the ensemble spread, the proofs are more technical. We concentrate on that case. The analogous results for \eqref{equ:homotopy_ode} follow by performing very similar, often nearly identical, calculations, but for fewer terms.

First of all, we prove in the following proposition that the objects of interest, $\theta^j_s$, $\mu_{\theta_s}$, $\eta^j_s$, and $\mu_s$ are well-defined.
\begin{proposition}
    The mean-field PDE \eqref{equ:second_order_mf_pde} has a unique global solution $\mu_s$. Furthermore, the IPS \eqref{equ:aldiode} and the mean-field IPS \eqref{equ:aldi mf particles} also posses unique global solutions. 
\end{proposition}
The proposition is proven in Section \ref{sec:existence}, in Proposition \ref{prop:mfl_odes_existence} and Proposition \ref{prop:ips_existence}.
We are now in a position to state our main theorem.
\begin{theorem}
  Let $\{\theta_s^j\}_{j=1}^J$ be the solution to \eqref{equ:aldiode} with associated empirical measure $\mu_{\theta_s}$ and  let $\mu_s$ be the solution to \eqref{equ:second_order_mf_pde}. Then, there exists a constant $C_T$ such that
  \[ \mathbb{E} [\mathcal{W}_2 (\mu_{\theta_s}, \mu_s^N)] \leqslant
     C_{T} J^{- \frac{1}{2}}, \]
     for any $s \leq T$. Here, $\mu_s^N$ is the $N$-fold product measure of $\mu_s$ with itself.
  \label{thm:mfl_deterministic_aldi}
\end{theorem}

\begin{proof}
  (Sketch)
  We introduce an artificial mean-field particle system $\eta^j$ as described
  in Section \ref{sec:mfl}. The precise mean-field ODEs can be
  found in \eqref{equ:aldi mf particles}. We couple the $\theta_s^j$ to the
  $\eta_s^j$ by choosing the same initial conditions, i.e., $\eta^j_0 =
  \theta^j_0$. Since the Wasserstein-distance is an infimum over all couplings, we can bound the Wasserstein distance of the empirical measures by evaluating the right hand side of \eqref{eq:wasserstein_definition} for this specific coupling:
  \begin{equation} \mathbb{E} [\mathcal{W}_2 (\mu_{\theta_s}, \mu_{\eta_s})] \leqslant
     \mathbb{E} \left[ \left( \frac{1}{J} \sum_{j = 1}^J | \theta^j_s -
     \eta^j_s |^2 \right)^{1 / 2} \right].
     \label{equ:w2_upper_bound_coupling}
  \end{equation}
  We fix a $T$ and assume that $s \leqslant T$. To bound \eqref{equ:w2_upper_bound_coupling}, we define 
  \[
    \Delta_s = \left( \frac{1}{J} \sum_{j =  1}^J | \theta^j_s - \eta^j_s |^2 \right)^{1 / 2}
  \]
  and upper bound its  growth. In Proposition \ref{prop:apriori_moment_bound} we will show that we can bound any moment of $\mathbb{E} [\Delta_s]$ independently of the ensemble size $J$, i.e.,
  \begin{equation}
    \mathbb{E}\left[|\Delta_s|^p\right]^{1/p} \lesssim C_p,
    \label{equ:holds_for_alpha_0}
  \end{equation}
  where $x \lesssim y$ symbolizes that the inequality $x \leqslant a y$ holds with a constant $a$ depending only on $T, m_0, P_0, \Phi$ and $d$. 
  
  We then use a bootstrapping technique inspired by \cite{ding2021ensemble}. The main idea is to show that, if
  \begin{equation}
    \mathbb{E}[\Delta_s] \lesssim J^{-\gamma}
    \label{equ:a_priori_assumption}
  \end{equation}
  for some $\gamma \ge 0$, then one can actually improve that $\gamma$ value to a better $\gamma'$, i.e., 
  \begin{equation*}
    \mathbb{E}[\Delta_s] \lesssim J^{-\gamma'}
  \end{equation*} 
  holds for some $\gamma' > \gamma$. Plugging $p=0$ into \eqref{equ:holds_for_alpha_0}, we see that \eqref{equ:a_priori_assumption} holds for $\gamma = 0$. We then iteratively improve $\gamma = 0$ to any $\gamma = \frac{1}{2} - \epsilon$.

  We now go into a bit more detail on how to improve the current estimate of $\gamma$. We denote by $H_{\alpha}$ a random variable, which can change from occurrence to occurrence, such that $\mathbb{E} [H_{\alpha}^p]^{1 / p} \leqslant C_p J^{- \alpha}$ for all $p \ge 0$, where $C_p$ only depends on $T, m_0, P_0, \Phi, d$ and $p$. Considering the ODE for $\Delta_s$ and bounding all terms of its right-hand side, we end up with
  \begin{equation}
  \begin{aligned}
    \frac{\mathd}{\mathd s} \mathbb{E} [\Delta_s] & \leqslant  C (\mathbb{E}
    [\Delta_s] +\mathbb{E} [\Delta_s H_{1 / 4}] +\mathbb{E} [H_{1 / 2}])\\
    & \leqslant  C (\mathbb{E} [\Delta_s] +\mathbb{E}
    [\Delta_s^{\varepsilon} \Delta_s^{1 - \varepsilon} H_{1 / 4}] +\mathbb{E}
    [H_{1 / 2}])\\
    & \leqslant  C (\mathbb{E} [\Delta_s] +\mathbb{E} [\Delta_s]^{1 -
    \varepsilon} \mathbb{E} [\Delta_s  (H_{1 / 4})^{1 /
    \varepsilon}]^{\varepsilon} + \mathbb{E}[H_{1 / 2}])\\
    & \leqslant  C (\mathbb{E} [\Delta_s] +\mathbb{E} [\Delta_s]^{1 -
    \varepsilon} \mathbb{E} [\Delta_s^2]^{\varepsilon / 2} \mathbb{E} [(H_{1 /
    4})^{2 / \varepsilon}]^{\varepsilon / 2} + \mathbb{E}[H_{1 / 2}])).
  \end{aligned}
  \label{eq:eps_hoeldering}
  \end{equation}
  Here we repeatedly used H\"older's inequality. Now, all moments of $H_\alpha$ can be bounded by $J^{-\alpha}$
  \begin{align*}
    \frac{\mathd}{\mathd s} \mathbb{E} [\Delta_s] & \lesssim  \mathbb{E}
    [\Delta_s] +\mathbb{E} [\Delta_s]^{1 - \varepsilon} J^{- 1 / 4} + J^{- 1 /
    2}\\
    & \leqslant  \mathbb{E} [\Delta_s] + J^{- \gamma (1 - \varepsilon) - 1 /
    4} + J^{- 1 / 2}.
  \end{align*}
  In the second inequality we used the a priori assumption that $\mathbb{E}[\Delta_s] \leqslant J^{-\gamma}$.
  We now apply Groenwall to obtain
  \[ \mathbb{E} [\Delta_s] \lesssim J^{- \gamma'}, \]
  where $\gamma' = \min \left( \frac{1}{2}, \frac{1}{4} + \gamma (1 -
  \varepsilon) \right)$. Plugging in $\gamma = 0$, we obtain a rate of $J^{- 1 / 4}$ by applying the above argument once. Iterating the argument, we can achieve any rate smaller than $\alpha = \frac{1}{2}$. By finally applying the argument one more time, we can achieve $\alpha = \frac{1}{2}$.
  The full details of the proof can be found in Appendix \ref{sec:mfl_deterministic_aldi_proof}.
\end{proof}

\noindent
The following a priori bound is crucial for the proof of Theorem \ref{thm:mfl_deterministic_aldi}.
\begin{proposition}
  \label{prop:apriori_moment_bound} For $s \in [0, T]$,
  \[ \mathbb{E}\left[|\frac{1}{J} \sum | \theta^j_s - \eta^j_s |^2 |^p\right]^{1/p} \leqslant
     C_p, \]
  i.e., the $p$-norm can be bounded independently of $J$ for a fixed $T$. The constant $C_p$ only depends on $m_0, P_0, \Phi, d$ and $T$. 
\end{proposition}
The proof relies on the fact that $\sum_{j=1}^J |\theta_0^j - \eta_0^j|$ is a martingale and uses martingale inequalities. It can be found in Section \ref{sec:apriori_moment_bound}.

\section{Algorithmic implementation}\label{sec:time_stepping}

 A typical way of time-stepping the interacting particle systems presented in this paper is the forward-Euler method. However, due to its restricted domain of stability, using this method can lead to restrictions on the step-size $\Delta s$. In this section we describe tamed discetizations for the homotopy based moment matching formulation \eqref{equ:homotopy_ode} and the deterministic second-order dynamical sampler \eqref{equ:aldiode}. 
 We introduce a step-size $\Delta s \geq 0$ and discrete times $s_k = k \Delta s$. Further, we use the shorthand $\theta_{s_k} \approx \theta_k, m_{\theta_{s_{k}}} \approx m_{k}, P_{\theta_{s_{k}}} \approx P_{k}, \Theta_{s_{k}} \approx \Theta_{k}, \mu_{\theta_{s_{k}}} \approx \mu_{k}$ in the forthcoming subsections.

\subsection{Homotopy using moment matching}
We employ modifications to the time stepping for moment-matching method \eqref{equ:homotopy_ode} by using the following tamed discretizations
\begin{equation}\label{eqn:hom_theta_stepping}
    \theta_{k+1}^j = \theta_k^j - \frac{\Delta s}{2} P_k \Phi \left( M_k\Phi^{\rm T} (\theta_k^j -m_k) +  2 (\mu_{k}[y] - d) \right)
\end{equation}
where
\begin{equation*}
    P_k = \frac{1}{J} \Theta_k \Theta_k^{\rm T}
\end{equation*}
and
\begin{equation}\label{eqn:K}
M_{k} = \left(\Delta s \Phi^{\rm T} P_{k} \Phi +  \mu_{k} [R]\right)^{-1}
\end{equation}
for $j = 1, \ldots, J$.
As discussed in Section \ref{sec:homotopy_som}, we propagate $\theta_k^j$ forward by evolving the associated empirical mean and ensemble deviations using \eqref{eqn:hom_mean_ode}-\eqref{eqn:hom_dev_ode}. The resulting time-stepping of \eqref{eqn:hom_mean_ode}-\eqref{eqn:hom_dev_ode} is of the form
\begin{equation}\label{eqn:sec_mean_stepping}
    m_{k+1} = m_{k} - \Delta s P_{k} \Phi \left(\mu_{k}
			[y] - d \right)
\end{equation}
and 
\begin{equation}\label{eqn:sec_dev_stepping}
 \Theta_{k+1} = \Theta_{k} - \frac{\Delta s}{2}  P_{k} \Phi M_{k}\Phi^{\rm T} \Theta_{k} 
\end{equation} 
for the ensemble mean and ensemble deviations, respectively.

\begin{remark}
    Inverting the $N\times $N matrix (\ref{eqn:K}) can prove prohibitive for large data sets. However, since $R(\theta)$ is diagonal, taking the full inverse in (\ref{eqn:K}) can be replaced by inverting the diagonal entries of $\Delta s \Phi^{\rm T} P_{k} \Phi +  \mu_{k} [R]$ only, as proposed in \cite{https://doi.org/10.1002/qj.2186}. This inexpensive approximation still provides improved stability compared to an explicit Euler discretization of \eqref{eqn:hom_dev_ode}.
\end{remark}

\noindent
We provide pseudo-code summarizing the second-order moment matching method described in Algorithm \ref{table:algorithm_homotopy}.\\

    \begin{table}[H]
    \centering
    \refstepcounter{table}\label{table:algorithm_homotopy}
    \begin{tabular} {p{\linewidth}}
     \hline 
     \textbf{Algorithm 1: Homotopy using moment matching method for Bayesian inference} \\ [0.5ex] 
     \specialrule{1.5pt}{0pt}{0pt}
     \vspace{0.1ex}
     \textbf{Inputs}: Data set $\{(\phi^n, d^n)\}_{n = 1}^{N}$; feature map $\Phi = \{\phi^n\}_{n = 1}^{N}$; initial ensemble $\{\theta_ {0}^{j}\}_{j = 1}^{J}$ drawn from a Gaussian distribution; step-size $\Delta s$ and $K$ such that $\Delta s K=1$.\\[1ex]
    
    \textbf{for $k=0$ to $K-1$ do:} \\[0.2pt]
    \hfill\begin{minipage}{\dimexpr\textwidth-1mm}

        \begin{enumerate}[i]
            \item Evaluate ensemble mean $m_k$ \eqref{eqn:emp_mean}, ensemble deviations $\Theta_k$ \eqref{eqn:ens_dev}, covariance matrix $P_k$ \eqref{eqn:emp_cov}, $y$ \eqref{eqn:y}, $\mu_{k}
			[y]$, $\mu_{k}[R]$ and $M_k$ \eqref{eqn:K}.
            \item Evolve $m_k$ and $\Theta_k$ using \eqref{eqn:sec_mean_stepping}-\eqref{eqn:sec_dev_stepping}.
            \item Determine $\{\theta_k^j\}_{j=1}^J$ from $m_{k+1}$ and
            $\Theta_{k+1}$.\\
        \end{enumerate}
        
    \xdef\tpd{\the\prevdepth}
    \end{minipage}
        
    \textbf{end for}\\
    \textbf{Output:} Final ensemble 
    $\{\theta^{j}_K \}_{j=1}^J$.\\[1ex]
    \hline
    \end{tabular}
    \end{table}

\subsection{Deterministic second-order dynamical sampler}
We employ the idea of Trotter splitting to solve (\ref{equ:aldiode})
numerically. The required splitting is provided by
the evolution equation (\ref{equ:homotopy_ode}), already encountered in the homotopy approach, and the remainder
\begin{equation}\label{eqn:aldiode_split}
\frac{\mathd \theta_s^j}{\mathd s}  = -\frac{1}{2}P_{\theta_s} P_{\rm{prior}}^{-1}(\theta^j_s + m_{\theta_s} - 2 m_{\rm{prior}})
                                         + \frac{1}{2}(\theta^j_s -m_{\theta_s}). 
\end{equation}
Therefore, for every time step $k$, given $\{\theta_{k}^j\}_{j=1}^J$, we first compute $\theta_{k+1/2}^j$ using the second order moment matching method \eqref{eqn:hom_theta_stepping} rewritten as
\begin{equation*}\label{eqn:som_stepping}
    \theta_{k+1/2}^j = \theta_k^j - \frac{\Delta s}{2} P_k \Phi \left( M_k \Phi^{\rm T} (\theta_k^j -m_k) + 2( \mu_{k}[y] - d )\right)
\end{equation*}
for $j = 1,2,...,J$. Equivalently, we can obtain $\theta_{k+1/2}^j$ by evaluating $m_{k+1/2}$ and $\Theta_{k+1/2}$ as stated in \eqref{eqn:sec_mean_stepping}-\eqref{eqn:sec_dev_stepping} with subscript $k+1$ replaced by $k+1/2$. In the second half step, we approximate (\ref{eqn:aldiode_split}) using the following scheme:
\begin{equation}\label{eqn:aldi_stepping}
\begin{split}
    \theta_{k+1}^{j} & = \theta_{k+1/2}^{j} -\frac{\Delta s}{2}  P_{k+1/2} \left(\Delta s P_{k+1/2} + P_{\rm{prior}} \right)^{-1} (\theta_{k+1/2}^{j} +  m_{k+1/2} -2 m_{\rm{prior}})  \\
    & \quad +  \frac{\Delta s}{2}(\theta^j_{k+1/2} -m_{k+1/2}).
    \end{split}
\end{equation}
Again, if $P_{\rm prior}$ is diagonal, full matrix inversion in (\ref{eqn:aldi_stepping}) can be replaced by inverting the diagonal only.

\noindent
We provide pseudo-code describing the algorithm for deterministic second-order dynamical sampler in Algorithm \ref{table:algorithm_aldi}.\\

\begin{table}[H]
\centering
\refstepcounter{table}\label{table:algorithm_aldi}
\begin{tabular} {p{\linewidth}}
 \hline 
 \textbf{Algorithm 2: Deterministic second-order dynamical sampler for Bayesian inference} \\ [0.5ex] 
 \specialrule{1.5pt}{0pt}{0pt}
 \vspace{0.1ex}
\textbf{Inputs}: Data set $\{(\phi^n, d^n)\}_{n = 1}^{N}$; 
 feature map $\Phi = \{\phi^n\}_{n = 1}^{N}$;
 initial ensemble $\{\theta_{0}^{j}\}_{j = 1}^{J}$ drawn from a Gaussian distribution; step-size $\Delta s$; threshold value $\epsilon > 0$.\\[1ex]

\textbf{Initialize}: $k = 0$. \\

\textbf{while} $ \frac{\| P_{k+1} - P_k \|_2}{\| P_k \|_2} \geq \epsilon$ \textbf{do:} \\[0.2pt]
\hfill\begin{minipage}{\dimexpr\textwidth-1mm}
    \begin{enumerate}[i]
        \item Numerically solve ODE \eqref{equ:aldiode} by Trotter splitting:
            \begin{enumerate}
            \item Evaluate ensemble mean $m_k$ \eqref{eqn:emp_mean}, ensemble deviations $\Theta_k$ \eqref{eqn:ens_dev}, covariance matrix $P_k$ \eqref{eqn:emp_cov}, $y$ \eqref{eqn:y}, $\mu_{k}[y]$, $\mu_{k}[R]$ and $M_k$ \eqref{eqn:K}.
                \item Determine $\{\theta^j_{k+1/2}\}_{j=1}^J$ by evolving $m_k$ and $\Theta_k$ using the time-stepping  \eqref{eqn:sec_mean_stepping}-\eqref{eqn:sec_dev_stepping}.
                \item Evaluate covariance matrix $P_{k+1/2}$ \eqref{eqn:emp_cov}.
                \item Determine $\{\theta^j_{k+1}\}_{j=1}^J$ using the time-stepping (\ref{eqn:aldi_stepping}).
            \end{enumerate}
        \item Update the ensemble $\{\theta_{k}^j\}_{j=1}^J \rightarrow \{\theta_{k+1}^j\}_{j=1}^J$.
    \end{enumerate}
    Increment $k$.
    \vspace{0.6em}
\end{minipage}
\textbf{end while}\\
\textbf{Output:} Final ensemble $\{\theta^{j}_k \}_{j=1}^J$.\\[1ex]
\hline
\end{tabular}
\end{table}

\section{Numerical example}
In this section, we describe an experiment designed to evaluate the performance of the homotopy using moment matching method and deterministic second-order dynamical sampler for estimating model parameters in logistic regression. We provide a numerical example on a toy problem that demonstrates the applicability of the proposed methods to logistic regression with known true parameter value. We generate a reference parameter vector $\theta_{\text{ref}}$ of dimension $D=20$ by sampling from a standard normal distribution. Next, we generate $N=300$ data points $x_n \in \mathbb{R}^{\rm D}$, also sampled from a standard normal distribution. We consider a simple feature map $\Phi(x) = x$. For each data point $x_n$, we compute the probability $y_{x_n}(\theta_{\text{ref}})$ as defined in \eqref{eqn:prob}. All points are assigned labels $d_n = 1$ with probability $y_{x_n}(\theta_{\text{ref}})$ and $d_n = 0$ otherwise.

First, we test the proposed algorithms using a non-informative Gaussian prior with $m_{\rm{prior}} = 0$ and $P_{\rm{prior}} = I$. Additionally, we implement the inference methods using a non-diagonal prior covariance matrix $P_{\rm{prior}}$, defined as a random symmetric positive-definite matrix, generated using the scikit-learn library. This ensures that the prior is well-defined and captures potential correlations between parameters.
In both cases, we vary the ensemble size $J$ between $J=10$ and $J=100$. 

We report the $l_2$-difference between the true parameter vector $\theta_{\text{ref}}$ and the posterior ensemble mean, averaged over 100 repeated experiments and its standard deviation for Algorithms 1 and 2 in Tables \ref{tab:results_known_params_id} \& \ref{tab:results_known_params_non}. Results corresponding to $P_{\rm{prior}} = I$ are displayed in Table \ref{tab:results_known_params_id} and for non-diagonal $P_{\rm{prior}}$ in Table \ref{tab:results_known_params_non}.


\begin{table}
    \centering
    \resizebox{\textwidth}{!}{%
    \begin{tabular}{c|c|c|c|c}
        \hline
        \textbf{Method / $J$} & \textbf{10} & \textbf{20} & \textbf{50} & \textbf{100}  \\
        \specialrule{1.5pt}{0pt}{0pt}
        \makecell{Homotopy using \\ moment matching method} & $1.518 \pm 0.017$ & $1.283 \pm 0.009$ & $0.814 \pm 0.002$ & $0.484 \pm 0.001$ \\
        \makecell{Deterministic second-order \\ dynamical sampler} & $0.895 \pm 0.004$ & $0.502 \pm 0.007$ & $0.422 \pm 0.003$ & $0.282 \pm 0.002$ \\
        \hline
    \end{tabular}
    }
    \caption{$l_2$-difference between true parameter values and the posterior ensemble mean with $P_{\rm{prior}} = I$, averaged over 100 experimental runs, and its standard deviation as a function of ensemble size $J$.}
    \label{tab:results_known_params_id}
\end{table}

\begin{table}
    \centering
    \resizebox{\textwidth}{!}{%
    \begin{tabular}{c|c|c|c|c}
        \hline
        \textbf{Method / $J$} & \textbf{10} & \textbf{20} & \textbf{50} & \textbf{100}  \\
        \specialrule{1.5pt}{0pt}{0pt}
        \makecell{Homotopy using \\ moment matching method} & $1.519 \pm 0.019$ & $1.438 \pm 0.008$ & $0.857 \pm 0.003$ & $0.495 \pm 0.002$ \\
        \makecell{Deterministic second-order \\ dynamical sampler} & $0.887 \pm 0.012$ & $0.504 \pm 0.005$ & $0.461 \pm 0.005$ & $0.287 \pm 0.003$ \\
        \hline
    \end{tabular}
    }
    \caption{$l_2$-difference between true parameter values and the posterior ensemble mean with non-diagonal $P_{\rm{prior}}$, averaged over 100 experimental runs, and its standard deviation as a function of ensemble size $J$.}
    \label{tab:results_known_params_non}
\end{table}
The low mean $l_2$-difference and its standard deviation indicate that the ensemble mean provides a good approximation of the true parameter vector $\theta_{\text{ref}}$. We also conclude from results in Tables \ref{tab:results_known_params_id} \& \ref{tab:results_known_params_non} that both the proposed methods are able to qualitatively reproduce the true parameter value. It is observed that the deterministic second-order dynamical sampler behaves rather well over the entire range of ensemble sizes.

%
%

\section{Application to Bayesian logistic regression in neural networks}\label{sec:llba}

Neural networks using ReLU activation functions are possibly the most widely used neural network architectures for classification. However, it has been proven that these networks exhibit arbitrarily high confidence far away from the training data when fitted by minimizing the negative log-likelihood $\psi$ or, equivalently, by approximating the maximum likelihood estimator (MLE), denoted by $\tilde{\theta}_{\rm MLE}$, as demonstrated in \cites{Hein_2019_CVPR, pmlr-v70-guo17a}. Thus, this architecture along with a MLE training scheme is not robust and does not provide any measure of uncertainty in the model's predictions. 

One way of obtaining predictive uncertainty is to place distributions over the weights of a neural network, which leads to Bayesian neural networks. The idea of replacing the MLE $\tilde{\theta}_{\rm MLE}$ by a posterior measure $\tilde{\pi}_{\rm post}$ over parameters $\tilde{\theta}$ therefore enables us to make better informed predictions and to know when our model predictions are not to be trusted. 
To understand how uncertainty might be expressed in this setting,  
we put a prior $\tilde{\pi}_{\rm prior}$ on the parameters $\tilde{\theta}$ which can be thought of as incorporating some prior knowledge and then refining it based on data to learn a posterior distribution. That is, after observing $\mathcal{D}$, we can get the posterior measure through Bayes formula \eqref{eqn:sampling}.

This Bayesian approach to neural networks introduces a certain degree of computational complexity as we now need to sample from $\tilde{\pi}_{\rm post}$ instead of minimizing $\Psi$ which can also be computationally expensive. Computational approximations to $\tilde{\pi}_{\rm post}$ have been enabled by advances in Markov chain Monte Carlo (MCMC) methods (see \cite{neal2011mcmc}). However, even today's most sophisticated MCMC methods are rendered impractical for deep neural network architectures. Therefore, to make the Bayesian approach tractable, we focus on a last layer Bayesian approximation that places a prior distribution only on the output layer's parameters. Thus, we decompose an $l$-layered ReLU network $G_{\tilde{\theta}}:\mathbb R^{n} \rightarrow \mathbb R$ into a feature map $\phi_{\hat{\theta}} : \mathbb R^{n} \rightarrow \mathbb R^{d}$ consisting of the first $l-1$ layers of the ReLU network and the output layer, i.e.,
$$
G_{\tilde{\theta}}(x) = \sigma (\langle \theta,\phi_{\hat{\theta}}(x)\rangle)
$$
with $\tilde{\theta} = (\hat{\theta}^{\rm T},\theta^{\rm T})^{\rm T}$.
One now first trains the complete network using a (regularised) MLE approach, which provides $\tilde{\theta}_{\rm MLE}$ and the associated trained feature map 
\begin{equation} \label{eqn:feature_map}
\phi(x) := \phi_{\hat{\theta}_{\rm MLE}}(x).
\end{equation}
Furthermore, upon defining the input features $\phi^n = \phi (x^n)$, $n=1,\ldots,N$, over the data set $\{(x^n,d^n)\}_{n=1}^D$, this architecture is now equivalent to a Bayesian logistic regression problem in the output layer parameters $\theta$ as discussed in detail in Section \ref{sec:logistic_regression}.

In this paper, we analyze the performance of ReLU networks in the case of binary classification. While the work in \cite{Kristiadi2020BeingBE} focuses on Laplace approximations for Bayesian inference, we employ algorithms for Bayesian logistic regression based on the methods proposed in Section \ref{sec:methods}. We demonstrate experimentally that our methods in conjunction with pre-trained deterministic ReLU networks provide desirable uncertainty estimates. However, it should be noted that the methods proposed in this paper are not limited to a ‘last layer' use, but can be easily extended to multiple layers or the entire network.

We use a 3-layer ReLU network with 20 hidden units at each layer and $D = 50$ units at the output layer in the subsequent numerical experiments. The data set is constructed by generating a 2D binary classification data using scikit-learn. The MLE estimator $\tilde{\theta}_{\rm MLE}$ is obtained using the PyTorch library. Figure \ref{fig:data_bin} depicts the above generated data set.

\begin{figure}
    \centering
    \includegraphics[scale= 0.7]{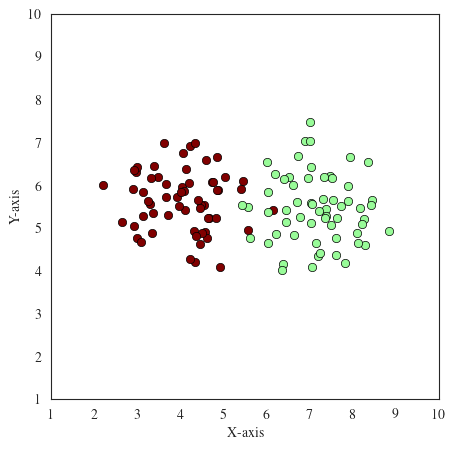}
    \caption{2D binary classification data set}
    \label{fig:data_bin}
\end{figure}

%
%

\section{Numerical experiments}\label{sec:results}

In this section, we consider a numerical experiment for binary classification problem as an illustration for uncertainty quantification in neural networks. We employ the already described 3-layer neural network architecture $G_{\tilde{\theta}} (x)$. The training data set $\{(x^n,d^n)\}_{n=1}^{N}$ consists of inputs and the associated labels.
We train the ReLU network using stochastic gradient descent (SGD) with 0.9 momentum, learning rate of $3 \times 10^{-4}$ and weight-decay for training over the 2D binary classification data set \ref{fig:data_bin}, using $N=30$ test points for toy binary classification problem. SGD minimizes the cross entropy loss quantity \eqref{eqn:cee} across the training data set, and we denote the computed minimizer by $\tilde{\theta}_{\rm MLE}$.

The computed set of parameters ($\hat{\theta}_{\rm MLE})$ is then used to provide the feature maps (\ref{eqn:feature_map}) which is used
for Bayesian logistic regression. The chosen prior is Gaussian  with mean $m_{\rm prior} = 0$ and covariance matrix $P_ {\rm prior}  = 2I$.
Using the ensemble of particles distributed according to the posterior $\tilde{\pi}_{\rm{post}}$, approximated with the homotopy based moment-matching method \eqref{eqn:sec_mean_stepping}-\eqref{eqn:sec_dev_stepping} and the second-order dynamical sampler \eqref{eqn:aldi_stepping}, the predictive distribution is estimated as
\begin{equation*}\label{eqn:pred_ensemble}
    \pi (d=1 |x, \mathcal{D} ) = \frac{1}{J}\sum_{j =1}^J \sigma(\langle\theta^j_{\ast}, \phi_{\hat{\theta}_{\rm{MLE}}}(x)\rangle ),
\end{equation*}
where $\{\theta^j_{\ast}\}_{j=1}^J$ is the final ensemble of particles obtained using Algorithm \ref{table:algorithm_homotopy} and Algorithm \ref{table:algorithm_aldi}. The associated Bayesian posterior distribution now translates uncertainty in weights to uncertainty in model predictions.


The results from Algorithms 1 \& 2 are compared to the following three alternative methods.

\textbf{Laplace approximations.}~We compare the uncertainty estimates provided by the proposed EnKF based methods for  Bayesian approximation over the output layer parameters to the last-layer Laplace approximation (LLLA) for inference as introduced in \cite{Kristiadi2020BeingBE}. In this case, we perform a Laplace approximation to get the posterior of the weights of the last layer, assuming the previous layers to be fixed at MLE estimates. So, for unknown output layer parameters $\theta$, we infer

\begin{equation}\label{eqn:posterior_laplace}
\pi (\theta | \mathcal{D} ) = \mathcal{N} (\theta | \theta_{\rm MLE}, H^{-1})
\end{equation}
where $H$ is the Hessian of the negative log-posterior with respect to $\theta$ at $\theta_{\rm MLE}$. The predictive distribution in the case of binary classification is thus given by
\begin{equation}\label{eqn:pred_dist}
    \pi (d=1 |x, \mathcal{D} ) = \int \sigma (\langle\theta, \phi(x)\rangle \pi(\theta | \mathcal{D}) \rm d\theta,
\end{equation}
where $\pi(\theta | \mathcal{D})$ is approximated with \eqref{eqn:posterior_laplace} and the integral \eqref{eqn:pred_dist} is computed using a probit approximation as described in  \cites{Kristiadi2020BeingBE, MacKay1992TheEF}.

\textbf{Ensemble learning.}~As a non-Bayesian method, we investigate the uncertainty estimates from ensemble learning (or deep ensembles), introduced in \cites{Lakshminarayanan2016SimpleAS, NEURIPS2019_8558cb40}. The technique uses an ensemble of deterministic networks, meaning that each network in the ensemble produces a point estimate rather than a distribution. We train an ensemble of $M = 5$ ReLU networks independently, using the entire training data set to train each network. Given an input point $x_n$, target label $d_n$ and cross-entropy loss $\Psi(\tilde{\theta}, x_n, d_n)$, we generate an adversarial sample 
$$
x_n^* = x_n + \xi \text{ sign}(\nabla_\theta \Psi(\tilde{\theta}, x_n, d_n)),
$$
where $ \xi \sim \mathcal{N}(0, 0.01)$. As described in \cite{Lakshminarayanan2016SimpleAS}, using adversarial samples for training by adding perturbations in the direction the network is likely to increase the loss provides a ‘better random' direction for smoothing predictive distributions.  For $M$ models with MLE estimated parameters $\{\tilde{\theta}_m\}_{m=1}^M$, we evaluate the ensemble predictions as 
\begin{equation*}
\pi(d=1 \mid x) = M^{-1} \sum_{m=1}^{M} \pi (d = 1 \mid x,\tilde{\theta}_m) 
= M^{-1} \sum_{m=1}^{M} G_{\tilde{\theta}_m}(x),
\end{equation*}
where $\{\tilde{\theta}_m\}_{m=1}^M$ denotes the parameters of the $m^{\rm{th}}$ model in the ensemble. For classification,  this corresponds to averaging the predicted probabilities.

\textbf{Hamiltonian Monte Carlo (HMC).}~A popular sampling-based approach for Bayesian inference in machine learning is the HMC \cite{neal2012bayesian}. We compare the uncertainty estimates obtained by the proposed EnKF based methods to HMC for inference over the output layer parameters.

We obtain results for the model's confidence in its prediction on and away from any input data $x$, where confidence is defined as the maximum predictive probability. In the case of a binary classification problem, confidence in prediction can be expressed as max$_{i \in \{0,1\}} \mathit \pi(d = i\mid x)$. Ideally, one would want the model to convey low confidence with such inputs. We next report on the results from these two experimental settings.

Results shown in Section \ref{sec:main_result} and Section \ref{sec:result_shift} corresponding to both methods introduced in this paper are reported using an ensemble size of $J=200$. Algorithm \ref{table:algorithm_homotopy} is implemented with step-size $\Delta s = 10^{-3}$ while Algorithm \ref{table:algorithm_aldi} is implemented with $\Delta s = 10^{-1}$ and $\Delta s K = 20$. For classification using HMC, we use 500 Monte Carlo samples across 7 independent chains for predictions.

\subsection{Predictive uncertainty in neural networks}\label{sec:main_result}

As shown in Figure \ref{fig:results_binary} and Figure \ref{fig:results_binary_zoom}, MLE predictions have very high confidence everywhere besides the region close to the decision boundary. Using deep ensembles of neural networks improves the accuracy of the prediction, which results in a better decision boundary and shows increased uncertainty estimates near the decision boundary. However, both MLE and ensemble learning predictions do not express uncertainty far away from the training data. On the other hand, last-layer Bayesian approximations, implemented with either Laplace, the homotopy moment matching method or second-order dynamical sampler, assign relatively high confidence close to the training data and are uncertain otherwise. In other words, the region of high confidence is identified much closer to the training data. All Bayesian approximations closely follow the decision boundary obtained using the MLE estimates and thus do not negatively impact the network's predictive accuracy. This is supported empirically by observing that all Bayesian methods maintain a classification accuracy of $96.67\%$ on test data, consistent with the MLE predictions. Thus, integrating Bayesian approximations into the final layer of an MLE-trained network does not compromise its original classification accuracy, a crucial aspect in neural network applications. Though similar uncertainty estimates can be seen for the second-order sampler and  HMC, the second-order sampler has a faster convergence towards equilibrium. Furthermore, the second-order dynamical sampler assigns higher confidence closer to training data than any other algorithm and allows one to use a larger step size $\Delta s$, and can therefore be recommended for further use. It can be noted that last-layer Laplace approximations also improve predictive uncertainty  but assigns lower confidence than our methods on and near the training data set.

\begin{figure}
  \centering
  \subfloat[MLE]{\includegraphics[width=0.42\textwidth]{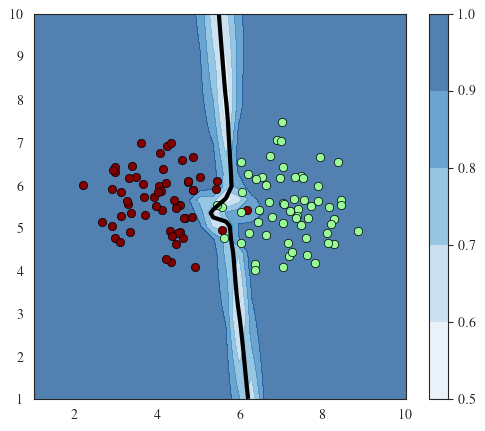}\label{fig:map}}
  \hfill
  \subfloat[Ensemble learning]{\includegraphics[width=0.42\textwidth]{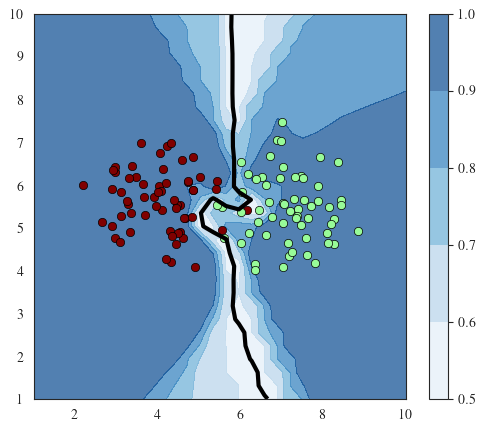}\label{fig:map_ens}}
  \hfill
  \subfloat[Last-layer Laplace]{\includegraphics[width=0.42\textwidth]{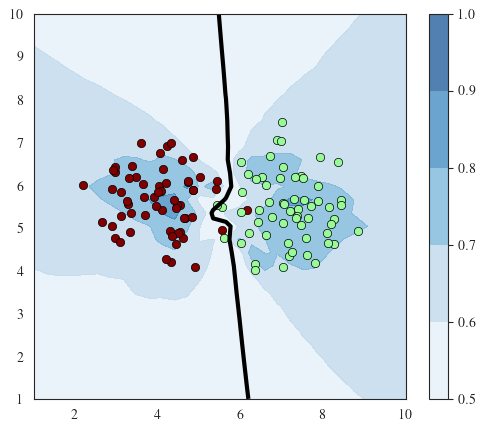}\label{fig:llla}}
  \hfill
  \subfloat[Last-layer HMC]{\includegraphics[width=0.42\textwidth]{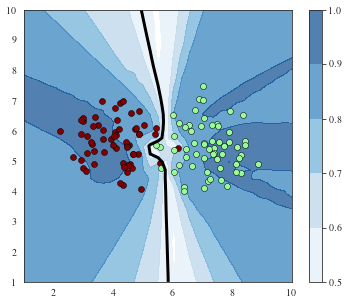}\label{fig:hmc}}
  \hfill
  \subfloat[Last-layer Moment matching method]{\includegraphics[width=0.42\textwidth]{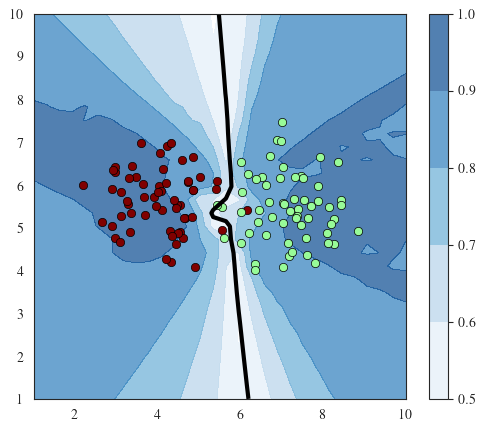}\label{fig:som_hom}}
  \hfill
  \subfloat[Last-layer deterministic second-order Dynamical Sampler]{\includegraphics[width=0.42\textwidth]{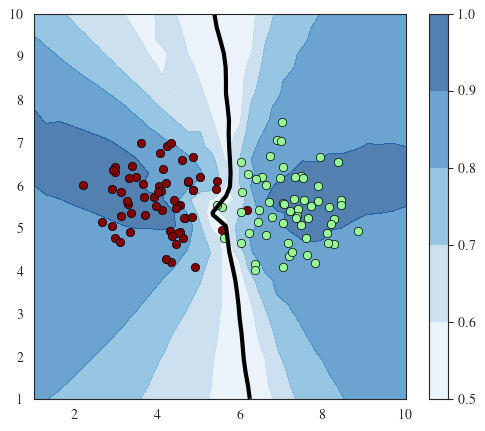}\label{fig:som_aldi}}
 \caption{Binary classification on a toy dataset using (a) MLE estimates, (b) ensemble of neural networks, last-layer Gaussian approximations over the weights obtained via (c) Laplace approximation, (d) Hamiltonian Monte Carlo  (e) moment matching method, (f) deterministic second-order dynamical sampler. Background colour depicts the confidence in classification while black line represents the decision boundary obtained for the toy classification problem.}
  \label{fig:results_binary}
\end{figure}

\begin{figure}
  \centering
  \subfloat[MLE]{\includegraphics[width=0.42\textwidth]{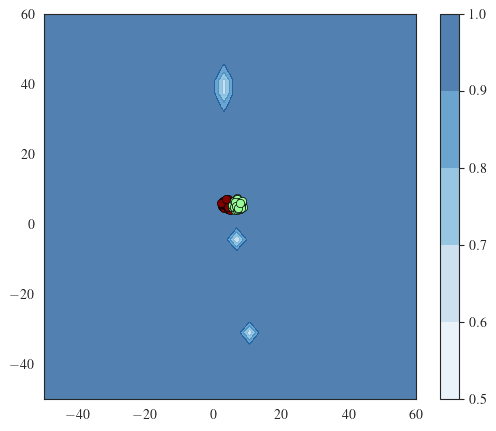}}
  \hfill
  \subfloat[Ensemble learning]{\includegraphics[width=0.42\textwidth]{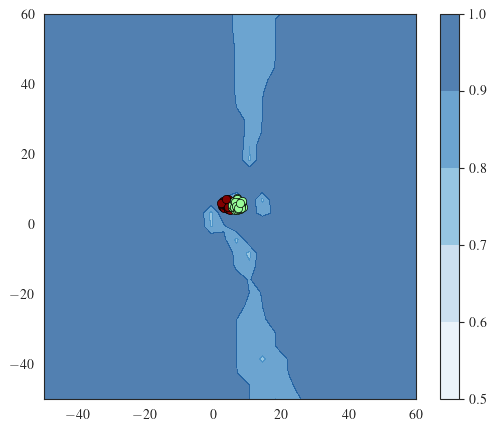}}
  \hfill
  \subfloat[Last-layer Laplace ]{\includegraphics[width=0.42\textwidth]{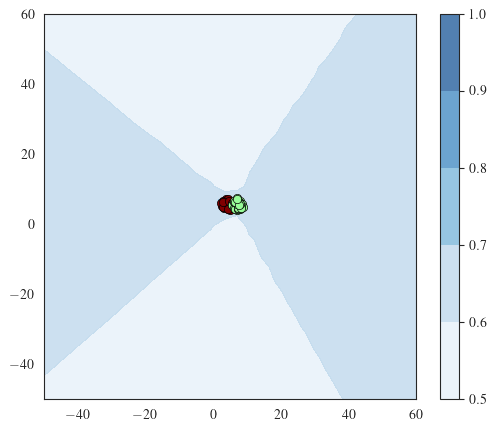}}
  \hfill
  \subfloat[Last-layer HMC]{\includegraphics[width=0.42\textwidth]{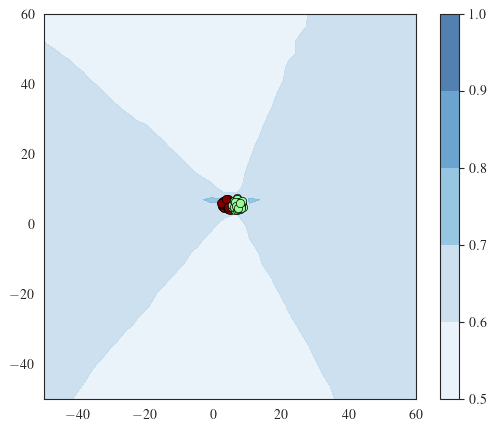}\label{fig:hmc_out}}
  \hfill
  \subfloat[Last-layer moment matching method]{\includegraphics[width=0.42\textwidth]{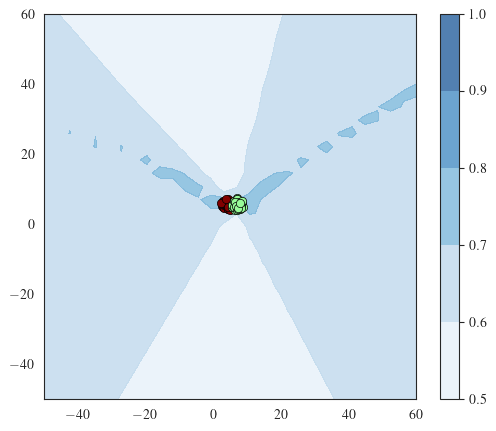}}
  \hfill
  \subfloat[Last-layer deterministic second-order dynamical sampler]{\includegraphics[width=0.42\textwidth]{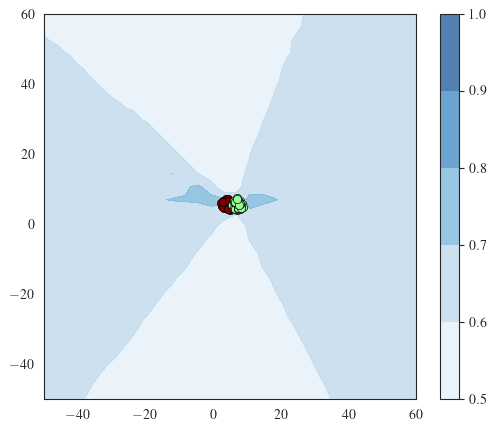}}
  \caption{Zoomed-out versions of the results in Figure 2 for binary classification on a toy data set using (a) MLE estimates, (b) ensemble of neural networks, last-layer Gaussian approximations over the weights obtained via (c) Laplace approximation, (d) Hamiltonian Monte Carlo  (e) moment matching method, (f) deterministic second-order dynamical sampler. Background colour depicts the confidence in classification.}
  \label{fig:results_binary_zoom}
\end{figure}

In Figure \ref{fig:results_binary_zoom}, we show a zoomed out version of the results in Figure \ref{fig:results_binary} to capture confidence levels significantly away from training data. It can be seen that the MLE estimate and ensemble learning demonstrate high confidence in the entire domain. The Bayesian approximations, even when applied to just the last-layer of a neural network, give desirable uncertainty estimates.

\subsection{Uncertainty on out-of-distribution data}\label{sec:result_shift}

In this experiment, we generate test samples on a large grid, such that the test points deviate significantly from the training samples. For each point in the test data set, we evaluate its Euclidean distance $\delta > 0$ from the nearest training point. We then evaluate the model's confidence in classifying these out-of-distribution (OOD) samples as a function of $\delta$. The results can be found as in Figure \ref{fig:results_conf}. It can be seen that as distance from the training set increases, the MLE technique is extremely overconfident in its predictions everywhere. MLE trained network always yields arbitrarily ‘overconfident' predictions away from training data and thus, is not robust. Using an ensemble of neural networks also does not improve uncertainty estimates and has little affect on the confidence on OOD data. However, last layer-Bayesian approximations assign lower confidence for OOD test data. As the distance from the training set increases, the model is less confident in its prediction and the level of confidence converges to a constant. Furthermore, our approaches assign maximum confidence (higher than Laplace approximations) when $\delta = 0$, i.e., the in-distribution data. 

\newpage
\begin{figure}
  \centering
  \subfloat[MLE]{\includegraphics[width=0.49\textwidth]{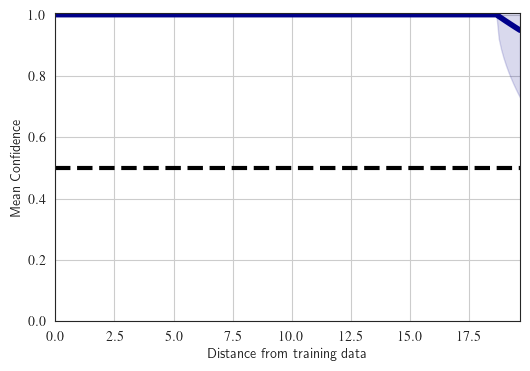}}
  \hfill
  \subfloat[Ensemble learning]{\includegraphics[width=0.49\textwidth]{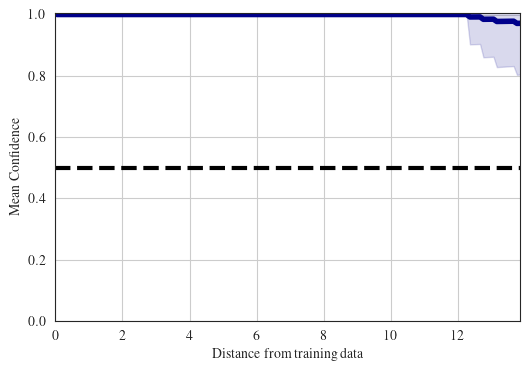}}
  \hfill
  \subfloat[Last-layer Laplace]{\includegraphics[width=0.49\textwidth]{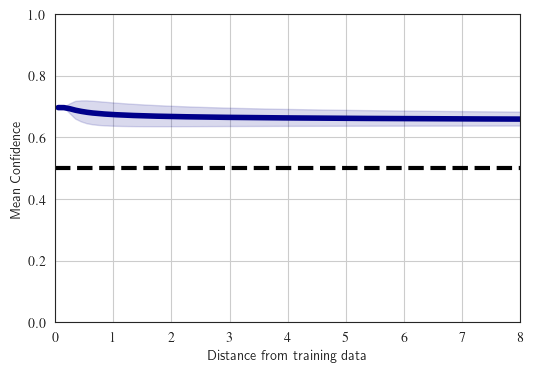}}
  \hfill
  \subfloat[Last-layer HMC]{\includegraphics[width=0.49\textwidth]{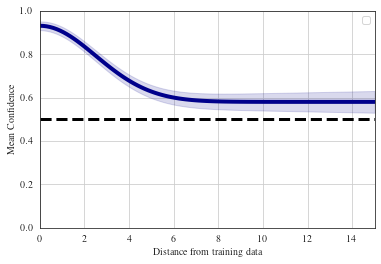}}
  \hfill
  \subfloat[Last-layer moment matching method]{\includegraphics[width=0.49\textwidth]{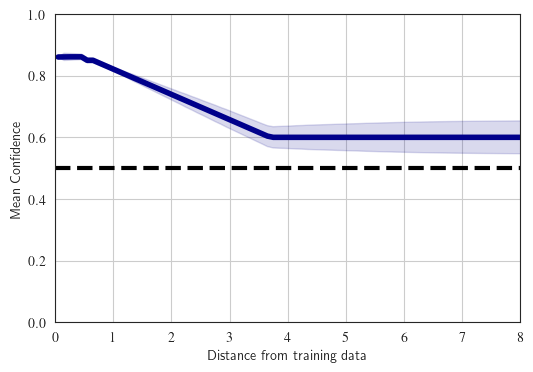}}
  \hfill
  \subfloat[Last-layer deterministic second-order dynamical sampler]{\includegraphics[width=0.49\textwidth]{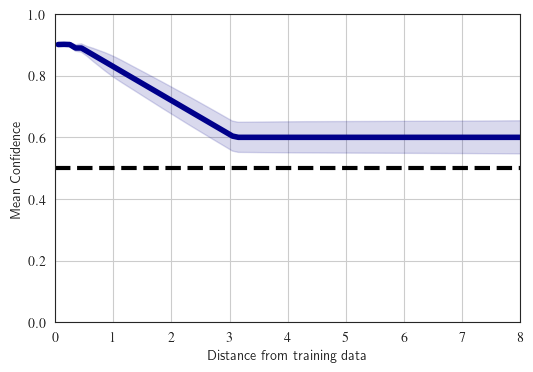}}
  
  \caption{Confidence of MLE, ensembles of neural networks, last-layer Laplace approximation, HMC, moment matching method, and deterministic second-order dynamical sampler as functions of $\delta$ over the test set. Thick blue lines and shades correspond to means and ± standard deviations, respectively. Dashed black lines signify the desirable confidence for $\delta$ sufficiently high. }
  \label{fig:results_conf}
\end{figure}

\subsection{Effect of varying ensemble sizes on predictive uncertainty}\label{sec:result_vary_ensemble}

We also analyze the effect of varying ensemble sizes on inference of last-layer network parameters.

\begin{figure}
    \centering
    
    \begin{subfigure}{0.19\linewidth}
        \includegraphics[width=\linewidth]{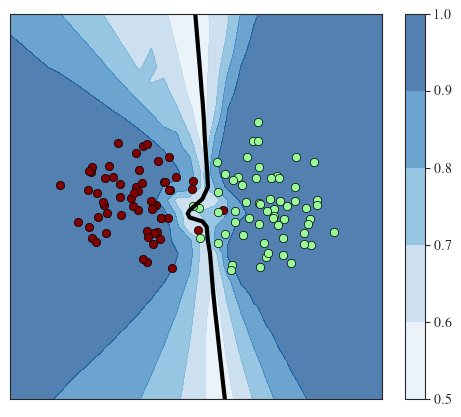}
        \caption{$J=30$}
    \end{subfigure}
    \begin{subfigure}{0.19\linewidth}
        \includegraphics[width=\linewidth]{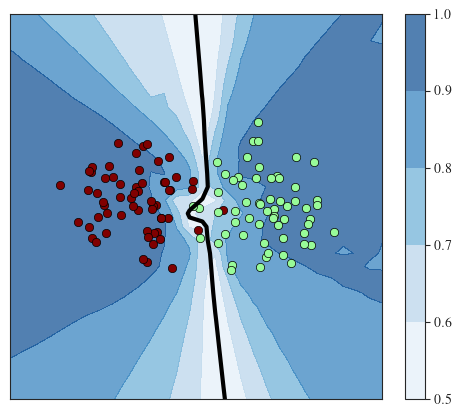}
        \caption{$J=50$}
    \end{subfigure}
    \begin{subfigure}{0.19\linewidth}
        \includegraphics[width=\linewidth]{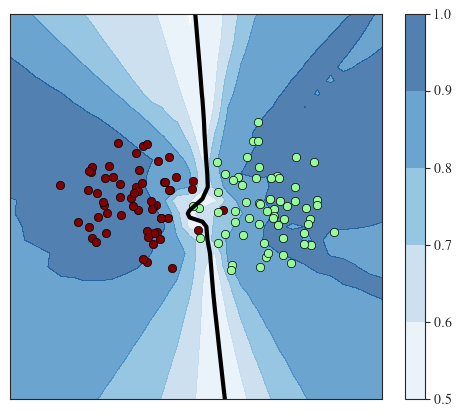}
        \caption{$J=100$}
    \end{subfigure}
        \begin{subfigure}{0.19\linewidth}
        \includegraphics[width=\linewidth]{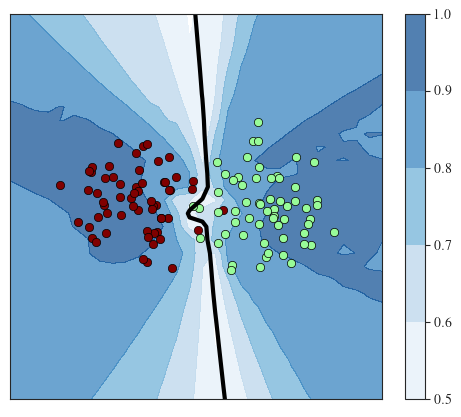}
        \caption{$J=200$}
    \end{subfigure}
    \begin{subfigure}{0.19\linewidth}
        \includegraphics[width=\linewidth]{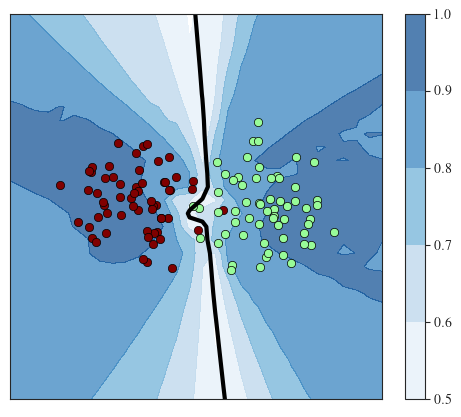}
        \caption{$J=300$}
    \end{subfigure}
    \captionsetup{justification=centering}
    \caption*{(i) Last-layer moment matching method}
    \label{fig:llsom}
    \medskip

\begin{subfigure}{0.19\linewidth}
        \includegraphics[width=\linewidth]{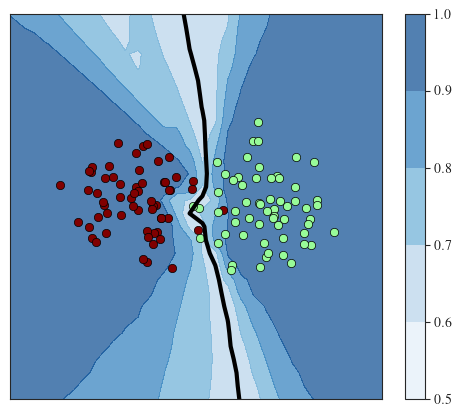}
        \caption{$J=30$}
    \end{subfigure}
    \begin{subfigure}{0.19\linewidth}
        \includegraphics[width=\linewidth]{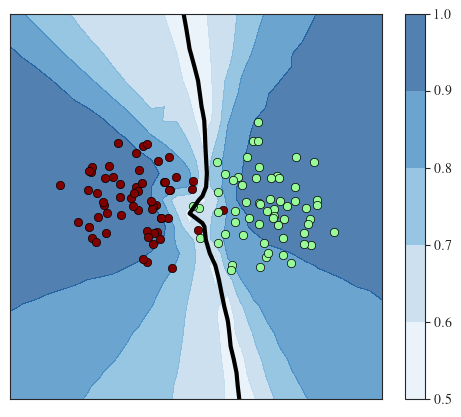}
        \caption{$J=50$}
    \end{subfigure}
    \begin{subfigure}{0.19\linewidth}
        \includegraphics[width=\linewidth]{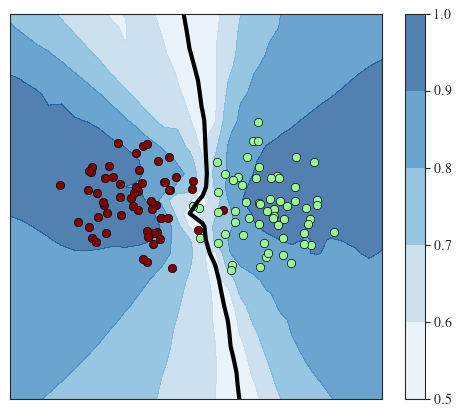}
        \caption{$J=100$}
    \end{subfigure}
    \begin{subfigure}{0.19\linewidth}
        \includegraphics[width=\linewidth]{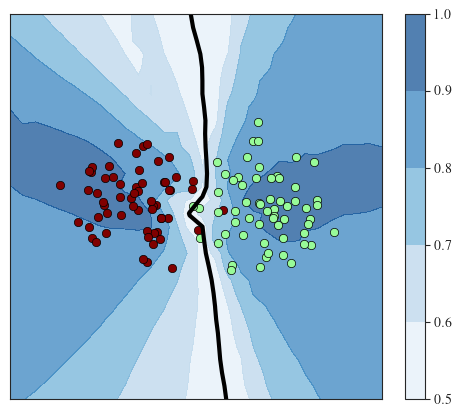}
        \caption{$J=200$}
    \end{subfigure}
    \begin{subfigure}{0.19\linewidth}
        \includegraphics[width=\linewidth]{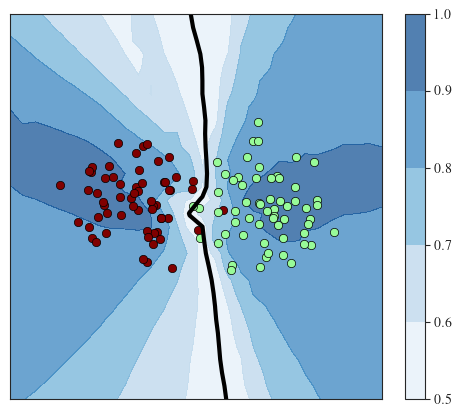}
        \caption{$J=300$}
    \end{subfigure}
    \captionsetup{justification=centering}
    \caption*{(ii) Last-layer deterministic second-order dynamical sampler}
    \label{fig:llmfds_det}
    \captionsetup{justification=raggedright}
    \caption{Effect of varying ensemble sizes $(J)$ on confidence in prediction for binary classification using proposed ensemble sampling methods for Bayesian inference over the network's output (last) layer.}
    \label{fig:results_ens}
\end{figure}

It can be observed in Figure \ref{fig:results_ens} that as ensemble size increases, the region of high confidence is identified a lot closer to the training data, while closely following the decision boundary  obtained by the MLE estimator. For parameters $D = 50$, as ensemble size $J \rightarrow \infty$ (here $J = 300$), we observe that the interacting particle systems converge to their mean-field limit. 
It can also be concluded that a large ensemble size $(J > D)$ provides better uncertainty estimates than small ensemble sizes $(J \leq D)$. However, using ensemble size  $(J < D)$ still results in better uncertainty estimates near the decision boundary than those obtained by MLE estimates.

\subsection{Extension to multiclass Bayesian logistic regression}\label{sec:multiclass}
In this section, we demonstrate that the EnKF-based methods of Bayesian approximation yield a similar behavior in multi-class settings. In the case of a problem arising from classification into $K$ classes, confidence in prediction can be expressed as max$_{k \in \{1, \dots, K\}} \mathit \pi(d = k\mid x)$. The posterior probabilities for $K$ classes $C_k$ is given by
\begin{equation}\label{eqn:prob_multi}
\mathbb{P}_\theta[\phi \in C_k] =\frac{e^{z_k}}{\sum_{j=1}^{K} e^{z_j}},
\end{equation}
where $z_j:=\theta_k^{\rm T}\phi.$
The predictive distribution in the case of multi-class classification is thus given by
\begin{equation}\label{eqn:pred_dist_multi}
    \pi (d=k |x, \mathcal{D} ) = \int \mathbb{P}_\theta[\phi \in C_k] \pi(\theta | \mathcal{D}) \rm d\theta.
\end{equation}
We obtain results for the model's confidence in its prediction on and away from any input data $x$, where confidence is defined as the maximum predictive probability in Figure \ref{fig:results_multi}. The data set is constructed by generating a 2D multi-class classification data using scikit-learn. We use a 3-layer ReLU network with 50 hidden units at each layer and $D = 50$ units at the output layer in the subsequent numerical experiment. The MLE estimator $\tilde{\theta}_{\rm MLE}$ is obtained using the PyTorch library. As shown in Figure \ref{fig:results_multi}, last-layer Bayesian inference, implemented with HMC, homotopy using moment matching method and second-order dynamical sampler, assign relatively high confidence close to the in-distribution data more than any other method and are uncertain otherwise. All Bayesian methods have a classification accuracy of $99\%$, same as the MLE predictions.
\begin{figure}
  \centering
  \subfloat[MLE]{\includegraphics[width=0.44\textwidth]{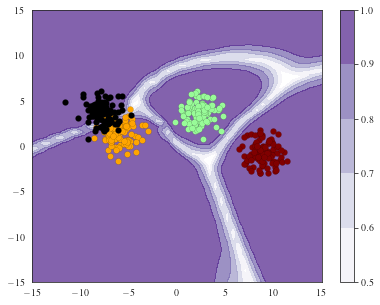}\label{fig:multi_map}}
  \hfill
  \subfloat[Ensemble learning]{\includegraphics[width=0.44\textwidth]{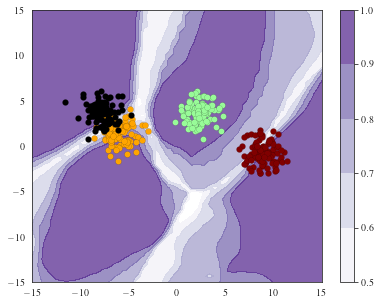}\label{fig:multi_map_ens}}
  \hfill
  \subfloat[Last-layer Laplace]{\includegraphics[width=0.44\textwidth]{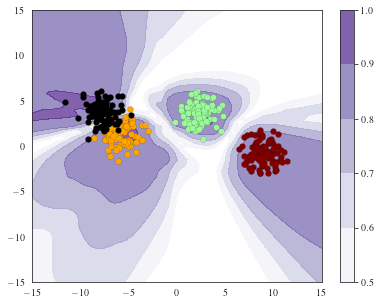}\label{fig:multi_llla}}
  \hfill
  \subfloat[Last-layer HMC]{\includegraphics[width=0.44\textwidth]{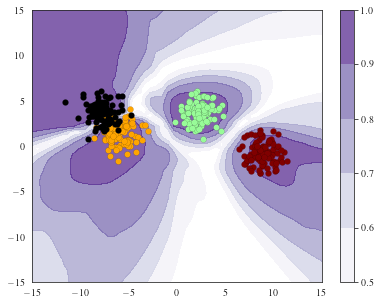}\label{fig:multi_hmc}}
  \hfill
  \subfloat[Last-layer moment matching method]{\includegraphics[width=0.44\textwidth]{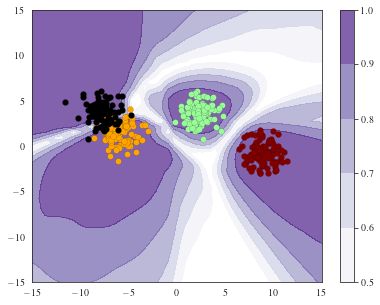}\label{fig:multi_som_hom}}
  \hfill
  \subfloat[Last-layer deterministic second-order dynamical sampler]{\includegraphics[width=0.44\textwidth]{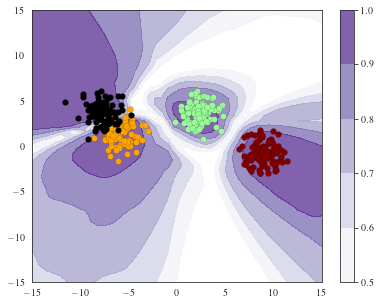}\label{fig:multi_som_aldi}}
 \caption{Multi-class classification on a toy dataset using (a) MLE estimates, (b) ensemble of neural networks, last-layer Gaussian approximations over the weights obtained via (c) Laplace approximation, (d) Hamiltonian Monte Carlo  (e) moment matching method, and (f) deterministic second-order dynamical sampler. Background colour depicts the confidence in classification obtained for the toy classification problem.}
  \label{fig:results_multi}
\end{figure}

\subsection{Quantifying uncertainty in deep learning-based image classification}\label{sec:deep_uq}
We perform binary classification on the CIFAR-10 dataset by exclusively utilizing the dog and cat classes among the ten available classes. A regularized ResNet10 deep neural network is trained for classification on this subset, termed the ‘dog-cat' dataset. A sample image is  labeled 0 if it belongs to the dog class and 1 if it is a cat. Consequently, the remaining classes within CIFAR-10 serve as the out-of-distribution (OOD) dataset. The model has $D=256$ parameters in the output layer. In Figure \ref{fig:results_cifar_binary}, we assess and report the model's confidence in prediction obtained using MLE training, Algorithms 1 \& 2, and the alternative approaches described in Section \ref{sec:llba} on both the in-distribution and OOD datasets.

Results shown in Figure \ref{fig:results_cifar_binary} corresponding to both methods introduced in this paper are reported using an ensemble size of $J=600$. Algorithm \ref{table:algorithm_homotopy} is implemented with step-size $\Delta s = 10^{-4}$ while Algorithm \ref{table:algorithm_aldi} is implemented with $\Delta s = 10^{-2}$ and $\Delta s K = 23$. For classification using HMC, we use 5000 Monte Carlo samples across 15 independent chains for predictions.

It was observed that all Bayesian methods achieved a classification accuracy of $98.01\%$, matching the performance of the MLE prediction. As illustrated in Figure \ref{fig:results_cifar_binary}, the MLE-trained model exhibited overconfident predictions on out-of-distribution (OOD) datasets. In contrast, Bayesian approximations applied to the final layer exhibited lower confidence levels for OOD data predictions. Previous results indicated similar uncertainty estimates for Bayesian inference using the second-order sampler, moment matching method, and Hamiltonian Monte Carlo (HMC). Therefore, we present the results for the second-order sampler exclusively in Figure \ref{fig:results_cifar_binary}. The deterministic second-order dynamical sampler assigned high confidence to in-distribution data and converged more rapidly than other algorithms. This faster convergence makes it a more computationally feasible option compared to HMC.

\begin{figure}
\begin{center}
  \subfloat[MLE]
  {\includegraphics[width=0.86\textwidth]{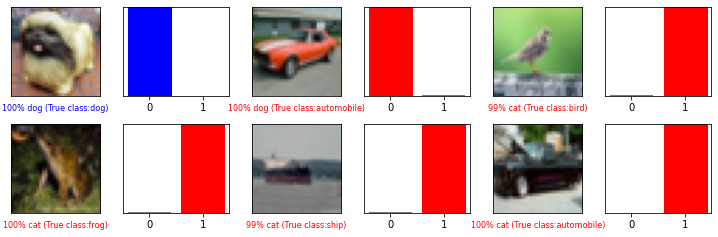}\label{fig:mle_binary}}
  \vspace{2mm}
  
  \subfloat[Ensemble learning]{\includegraphics[width=0.89\textwidth]{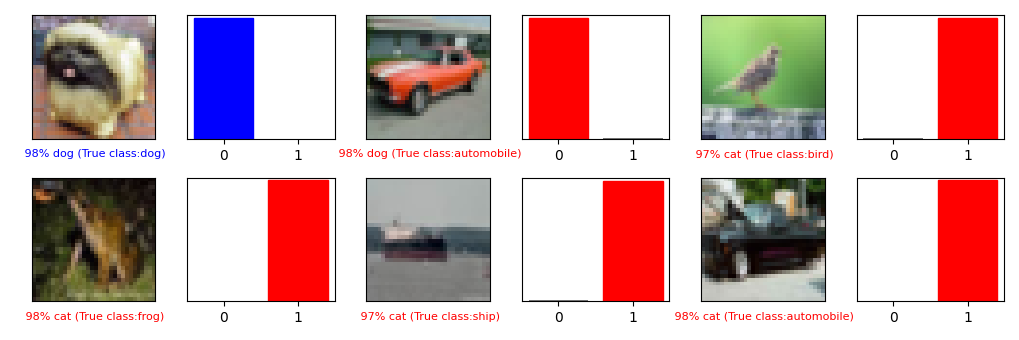}\label{fig:ens_map_binary}}
  \vspace{2mm}
  
  \subfloat[Last-layer Laplace]{\includegraphics[width=0.89\textwidth]{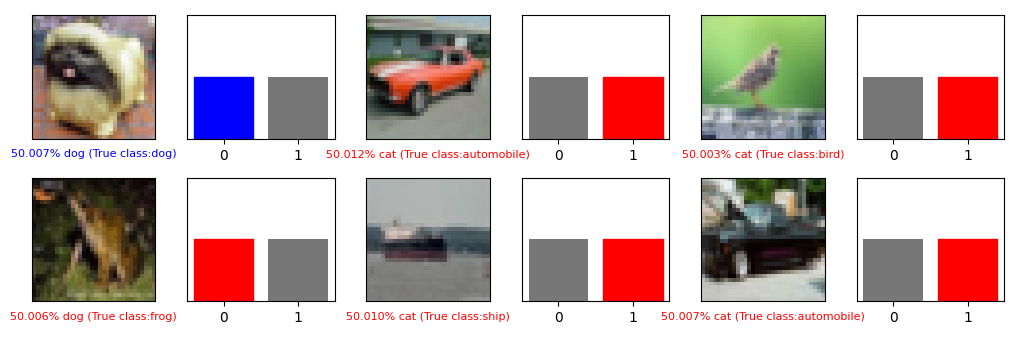}\label{fig:laplace_binary}}
  \vspace{2mm}
  
  \subfloat[Last-layer deterministic second-order dynamical sampler]{\includegraphics[width=0.87\textwidth]{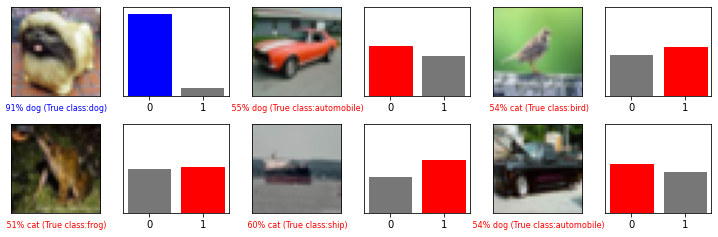}\label{fig:sampler_binary}}
 \caption{Binary classification on CIFAR10 dataset. Blue coloured barplots represent maximum probability of the test image belonging to one of the two classes on in-distribution test data and the red coloured barplots represent maximum probability for OOD test image.}
  \label{fig:results_cifar_binary}
\end{center}
\end{figure}

\newpage	
\section{Conclusions}
In this paper, we have presented two extensions of EnKF and related interacting particle systems to Bayesian logistic regression. We have proven quantitative convergence rates for these systems to their mean field limits as the number of particles tends to infinity. We have employed both methods for Bayesian inference in neural networks with  cross-entropy loss function. The numerical results confirm the effectiveness of the proposed methods for quantifying uncertainty. They have also shown that these uncertainty estimates make neural networks more robust with respect to distributional shifts.

\appendix

\section{Preliminary lemmas}

In this section, we prove some preliminary lemmas that we will need later on.

\begin{lemma}
  \label{lemma:high_moments_bound}Assume that $\xi_i$ are centred i.i.d. random
  variables that have moments of all orders. Then
  \begin{equation*}
   \left\| \sum_{j = 1}^J \xi_i \right\|_p \leqslant C_p  \| \xi_i \|_{2 p}
     J^{1 / 2} . 
  \end{equation*}
  In particular, it follows that
    \begin{equation*}
   \left\| \frac{1}{J} \sum_{j = 1}^J \xi_i \right\|_p \leqslant C_p \|
     \xi_i \|_{2 p} J^{- 1 / 2} 
    \end{equation*}
\end{lemma}

\begin{proof}
  We have the martingale
  \[ M_J = \sum_{j = 1}^J \xi_i . \]
  The statement of the lemma is now to bound the $L^p$ norms of $M_J$. To that end,
  \begin{align*} 
    \mathbb{E} [| M_J |^p]^{1 / p} 
    &\leqslant C_p \mathbb{E} [[M_J]^{p / 2}]^{1 / p} 
     \leqslant C_p \left( \mathbb{E} \left[ \left| \sum_{j = 1}^J
     \xi_i^2 \right|^p \right]^{1 / p} \right)^{1 / 2} \\
     &\leqslant C_p (J\mathbb{E} [\xi_i^{2 p}]^{1 / p})^{1 / 2} \leqslant C_p J^{1 / 2},
  \end{align*}
  where we applied the Burkholder-Davis-Gundy inequality, Jensen's inequality
  and the triangle inequality in that order. 
\end{proof}

Note that a similar result to Lemma \ref{lemma:high_moments_bound} has also been proven in \cite{ding2021ensemble}. However, by
employing martingale techniques, our proof is considerably shorter.
\begin{lemma}
		Let $f$ be a Lipschitz continuous function and $S$ be
		a linear subspace of $\mathbb{R}^D$. Then, for any $R > 0$, the map $(m, P)
		\rightarrow \mathbb{E}_{\mathcal{N} (m, P)} [f]$ is locally Lipschitz on
		\begin{equation}
            A = \{ (m, P) \in \mathbb{R}^D \times \tmop{PSD} (\mathbb{R}^D) | \ker
		(P) = S \}.
        \label{equ:local_lipschitz_A}
        \end{equation}
    \label{lemma:exp_local_lipschitz_N}
	\end{lemma}
	
	\begin{proof}
		For $(m, P) \in A$ we define the restriction of $P$ to $S^{\perp}$ as
		$\tilde{P}$. The map $\tilde{P}$ is invertible and therefore its smallest
		eigenvalue is greater than zero.
		
		Now let $(m_i, P_i) \in A, i = 1, 2$. Then there exists an $R$ such that
		$(m_i, P_i) \in A_R$, $i = 1, 2$, where
		\[ A_R = \left\{ (m, P) | \lambda_{\min} (\tilde{P} ) \geqslant
		\frac{1}{R}, \lambda_{\max} (P) \leqslant R, | m | \leqslant R \right\} .
		\]
		The importance of the lower bound on the eigenvalues is that it implies that
		$\| P_1^{1 / 2} - P_2^{1 / 2} \| = \| \tilde{P}_1^{1 / 2} - \tilde{P}^{1 /
			2}_2 \| \leqslant R \| \tilde{P}_1 - \tilde{P}_2 \| \leqslant R \| P_1 - P_2
		\|$, see Theorem 6.2. of ``Functions of Matrices, Theory and Computation''.
		Using that we see
		\begin{align*}
			&   | \mathbb{E}_{\mathcal{N} (m_1, P_1)} [f] -\mathbb{E}_{\mathcal{N}
				(m_2, P_2)} [f] |\\
			& =  \int | f (P_1^{1 / 2} (x + m_1)) - f (P_2^{1 / 2} (x + m_2)) |
			\mathcal{N} (x ; 0, I) \mathd x\\
			& \leqslant  f_L \| P_1^{1 / 2} - P_2^{1 / 2} \|  \int | x | \mathcal{N}
			(x ; 0, I) \mathd x + f_L | P_1^{1 / 2} m_1 - P_2^{1 / 2} m_2 |\\
			& \leqslant  Cf_L \| P_1^{1 / 2} - P_2^{1 / 2} \| + f_L  \| P_1^{1 / 2}
			- P_2^{1 / 2} \|  | m_1 | + f_L  \| P_2^{1 / 2} \|  | m_1 - m_2 |
		\end{align*}
		which shows that the difference is locally Lipschitz on $A_R$.
	\end{proof}

\begin{lemma}
  \label{lemma:covariance difference bound}Let $P_{\mu}$ and $P_{\nu}$ be the
  covariance matrices of two probability measures $\mu$ and $\nu$. Then
  \begin{align*}
    \| P_{\mu} - P_{\nu} \|_{\tmop{op}} & \leqslant  (\tmop{tr} (P_{\mu})^{1
    / 2} + \tmop{tr} (P_{\nu})^{1 / 2}) \mathbb{E} [| X -\mathbb{E} [X] - (Y
    -\mathbb{E} [Y]) |^2]^{1/2}
  \end{align*}
  where $(X, Y) \sim \gamma$ and $\gamma$ can be any coupling with marginals
  $X \sim \mu$ and $Y \sim \nu$.
\end{lemma}

\begin{proof}
  We denote the means of $\mu$ and $\nu$ by $m_{\mu}$ and $m_{\nu}$
  respectively. Then,
\begin{eqnarray*}
  P_{\mu} - P_{\nu} & = & \int (x - m_{\mu}) \otimes (x - m_{\mu}) \mathd \mu
  (x) - \int (y - m_{\nu}) \otimes (y - m_{\nu}) \mathd \nu (y)\\
  & = & \int (x - m_{\mu}) \otimes (x - m_{\mu} - (y - m_{\nu})) \mathd
  \gamma (x, y)\\
  & + &  \int (x - m_{\mu} - (y - m_{\nu})) \otimes (y - m_{\nu}) \mathd
  \gamma (x, y)
\end{eqnarray*}

  which holds for {any} coupling $\gamma$ of $\mu$ and $\nu$. Now, we can bound
\begin{align*}
  \| P_{\mu} - P_{\nu} \|  \leqslant & \int \| x - m_{\mu} \|  \| x - m_{\mu}
  - (y - m_{\nu}) \| \mathd \gamma (x, y)\\
   + & \int \| y - m_{\nu} \|  \| x - m_{\mu} - (y - m_{\nu}) \| \mathd
  \gamma (x, y)\\
   \leqslant & \left( \int \| x - m_{\mu} \|^2 d \mu (x) \right)^{1 / 2}
  \left( \int \| x - m_{\mu} - (y - m_{\nu}) \| ^2 \mathd \gamma (x, y)
  \right)^{1 / 2}\\
   + & \left( \int \| y - m_{\nu} \|^2 d \nu (y) \right)^{1 / 2} \left( \int
  \| x - m_{\mu} - (y - m_{\nu}) \| ^2 \mathd \gamma (x, y) \right)^{1 / 2}\\
   = & (\tmop{tr} (P_{\mu})^{1 / 2} + \tmop{tr} (P_{\nu})^{1 / 2}) \left(
  \int \| x - m_{\mu} - (y - m_{\nu}) \| ^2 \mathd \gamma (x, y) \right)^{1 /
  2}
\end{align*}

\end{proof}

\section{Notation for the proofs}\label{sec:notation_proofs}
Additional to the notation defined in Section \ref{sec:methods}, we will use the following notation:
\begin{enumerate}
  \item $| v |$ denotes the Euclidean norm of a vector $v \in \mathbb{R}^D$,
  
  \item $\| P \|_{\tmop{op}}$ denotes the operator norm of a matrix, i.e. $\|A\|_{\tmop{op}} = \sup_{|v| = 1} \frac{|A v|}{|v|}$,
  
  \item $\| P \|_{\text{Fr}}$ denotes the Frobenius norm of a matrix,
  
  \item $\| X \|_p$ denotes the $L^p$ norm of a random variable, $\| X \|_p
  =\mathbb{E} [| X |^p]^{1 / p}$
  
  \item $H_{\alpha}$ denotes a random variable with the property that all its
  moments decay like $J^{- \alpha}$, where $J$ is ensemble size. I.e., for any
  $p$, there is a $C_p$ such that $\| H \|_p \leqslant C_p J^{- \alpha}$. Note
  that $H_{\alpha}$ can change from occurrence to occurrence.

  \item The quantities $m_{\eta_s}$, $P_{\eta_s}$ are defined analoguous to $m_{\theta_s}$ and $P_{\theta_s}$ (see \eqref{eqn:emp_mean} and \eqref{eqn:emp_cov} as the empirical mean and covariance of the ensemble $\{\eta^j_s\}_{j=1}^J$. 

  \item The quantities $m_s$ and $P_s$ are the mean and covariance of $\mu_s = \mathcal{N}(m_s, P_s)$.

  \item Quantities like $y_{\theta_s}$ or $R_{\theta_s}$ denote the expectation of the functions $y$ or $R$ with respect to $\mu_{\theta_s}$, i.e.
  \[
    y_{\theta_s} = \mu_{\theta_s}[y]
  \]
  and similarly for $R$. The analogue holds for $y_{\eta_s}$ and $R_{\eta_s}$. $y_s$ and $R_s$ are the expectations with respect to $\mu_s$.
\end{enumerate}

\section{Proofs for the deterministic second-order dynamical sampler}

\subsection{Existence of solutions}\label{sec:existence}

We first prove that the mean-field PDE has a unique solution:
\begin{proposition}
  \label{prop:mfl_odes_existence}
    The mean-field PDE \eqref{equ:second_order_mf_pde} has a unique global solution $\mu_s$. 
\end{proposition}
\begin{proof} 
Since $\mu_0$ is Gaussian, and the right hand side of \eqref{equ:second_order_mf_pde} is linear for a fixed $\mu$, $\mu_s$ will stay Gaussian for all times. Therefore, the solutions to
\eqref{equ:second_order_mf_pde} can be obtained by solving the ODEs for the mean and
covariance of $\mu_s$, given in \eqref{equ:mflmeancovevolution}.

  Using Lemma \ref{lemma:exp_local_lipschitz_N} we see that the ODE is locally Lipschitz,
  which implies uniqueness and existence up until the exit time from $A$ as defined in \eqref{equ:local_lipschitz_A}. $A$
  can be left if $m$ or $P$ explode. However, we can see that
  \begin{align*}
    \frac{\mathd}{\mathd s} \| P_s \|_{\text{Fr}}^2 & =  \left\langle - P_s \Phi R_s
    \Phi^{\rm T}P_s - \frac{1}{2} (P_{\tmop{prior}}^{- 1} P_s + P_s
    P_{\tmop{prior}}^{- 1}) + \frac{1}{2} P_s, P_s \right\rangle_{\text{Fr}}\\
    & =  - \tmop{tr} (- P_s^{3 / 2} \Phi R_s \Phi^{\rm T}P_s^{3 / 2}) -
    \frac{1}{2} \tmop{tr} (P_s P_{\tmop{prior}}^{- 1} P_s) + \frac{1}{2} \|
    P_s \|_{\text{Fr}}^2 \leqslant \frac{1}{2} \| P_s \|_{\text{Fr}}^2
  \end{align*}
  where we used that $R_s, \Phi$ and $P_{\tmop{prior}}^{- 1}$ are globally
  bounded, the cyclic property of the trace, and the positive semi-definiteness
  and symmetry of $P_s$ and $R_s$. We can now apply Groenwall to see that the
  Frobenius-norm of $P_s$ does not explode. Since the operator norm is
  equivalent, the same holds for $\| P_s \|_{\tmop{op}}$:
  \begin{equation}
    \| P_s \|_{\tmop{op}} \lesssim \| P_s \|_{\text{Fr}} \leqslant e^s \| P_0 \|_{\text{Fr}} = e^s
    \| P_{\tmop{prior}} \| \label{equ:mflgrowthboundPs}
  \end{equation}
  The other way to exit the set A from \eqref{equ:local_lipschitz_A} is for the kernel of $P$ to
  change. The kernel cannot decrease, since if $v$ is in the kernel of $P_s$,
  then
  \[ \frac{\mathd}{\mathd s} P_s v = \left( - P_s \Phi R_s \Phi^{\rm T}P_s -
     \frac{1}{2} (P_{\tmop{prior}}^{- 1} P_s + P_s P_{\tmop{prior}}^{- 1}) +
     \frac{1}{2} P_s \right) v = 0 \]
  and therefore $v$ stays in the kernel of $P_s$. We now prove that the kernel
  cannot increase either. Let $(\lambda_s^i, v_s^i)$ are the eigenvalues and
  normalized eigenvectors of $P_t$. Then by \cite[Section 2.2]{dieci1999smooth},

  \begin{align*}
    \frac{\mathd}{\mathd s} \lambda_s^i & =  \left\langle v_s^i,
    \frac{\mathd}{\mathd s} P_s v_s^i \right\rangle\\
    & =  \left\langle v_s^i, -
    P_s \Phi R_s \Phi^{\rm T}P_s v_s^i - \frac{1}{2} (P_{\tmop{prior}}^{- 1} P_s +
    P_s P_{\tmop{prior}}^{- 1}) v_s^i + \frac{1}{2} P_s v_s^i \right\rangle\\
    & =  - (\lambda_s^i)^2 \langle v_s^i, \Phi R_s \Phi^{\rm T}v_s^i \rangle -
    \lambda^i_s \langle v_s, P_{\tmop{prior}}^{- 1} v_s \rangle + \frac{1}{2}
    \lambda_s^i\\
    & \gtrsim  - (\lambda_s^i)^2 - \lambda_s^i,
  \end{align*}
  where we used that the eigenvalues of $R_s$ are bounded from above. Therefore, by a
  differential inequality we get that $\lambda_s^i \geqslant \frac{c}{e^{\rm T}- c} > 0$,
  where $c$ is chosen such that $\frac{c}{1 - c} = \lambda_0^i$. This proves
  that no eigenvalue vanishes, and therefore the kernel of $P_s$ does not grow.
  Finally, for the mean we see that for $s \in [0, T]$
  \begin{align*}
    \frac{\mathd}{\mathd s} {| m_s |^2}  & = \langle - P_s \Phi (y_s - d) -
    P_s P_{\tmop{prior}}^{- 1} (m_s - m_{\tmop{prior}}), m_s \rangle\\
    & \lesssim  e^s (| m_s | + | m_s |^2) \leqslant e^s (1 + | m_s |^2),
  \end{align*}
  where we used $y_s$ and $d$ being bounded by $1$, and
  \eqref{equ:mflgrowthboundPs} to bound $\| P_s \|_{\tmop{op}}$. By applying
  Groenwall we now see that the norm of $m_s$ cannot explode either, and
  therefore the solution always stays in A \eqref{equ:local_lipschitz_A} and is defined globally.
  
  \ 
\end{proof}

\begin{proposition}
    The systems \eqref{equ:aldiode} and \eqref{equ:aldi mf particles} have unique global solutions.
    \label{prop:ips_existence}
\end{proposition}
\begin{proof}
    We start with \eqref{equ:aldiode}. Local existence of solutions is implied by local Lipschitzness. We now prove that the solutions do not explode and are therefore global. From \eqref{eqn:aldi_dev_ode} one can derive the evolution equation for $P_{\theta_s}$ as
    \begin{equation*}
        \begin{aligned}
            \frac{\mathd}{\mathd s} P_{\theta_s} 
            &= \frac{\mathd}{\mathd s} \Theta_s \Theta_s^{\rm T}
            = (\frac{\mathd}{\mathd s} \Theta_s) \Theta_s^{\rm T}+ \Theta_s \frac{\mathd}{\mathd s} \Theta_s^{\rm T}\\
            &= - P_{\theta_s} \Phi R_{\theta_s} \Phi^{\rm T}P_{\theta_s} - P_{\theta_s} P_{\text{prior}}^{-1} P_{\theta_s} + P_{\theta_s}.
        \end{aligned}
    \end{equation*}
    We then get that
    \begin{equation*}
    \begin{aligned}
        \frac{\mathd}{\mathd s} \|P_{\theta_s}\|^2_{\text{Fr}}
            &= -2\text{tr}\left(P_{\theta_s} \Phi R_{\theta_s} \Phi^{\rm T}P_{\theta_s}^2\right) 
            - 2\text{tr}\left(P_{\theta_s} P_{\text{prior}}^{-1} P_{\theta_s}^2\right)
            + 2\text{tr}\left(P_{\theta_s}^2\right) \\
            &\leqslant 2 \|P_{\theta_s}\|_{\text{Fr}}^2.
    \end{aligned}
    \end{equation*}
    Here we used the cyclic property of the trace $\text{tr}(AB) = \text{tr}{A^{1/2}BA^{1/2}}$ and the semi positive definiteness of $P_{\theta_s}, P_\text{prior}$ and $\Phi R_{{\theta_s}} \Phi^{\rm T}$.
    By using Grönwall's Lemma, we see that $\|P_{\theta_s}\|_{\text{Fr}}$ stays bounded on $[0,T]$. 
    We now prove that the mean also stays bounded. From \eqref{eqn:aldi_mean_ode}, 
    \begin{equation*}
    \|\frac{\mathd}{\mathd s} m_{\theta_s}\|  
    \leqslant - a\|P_{\theta_s}\|_{\text{op}} \left( 1 + (\|m_{\theta_s}\| +  1)\right),
\end{equation*}
    for some constant $a$. We used that $d, \Phi$, $P_\text{prior}$ and $m_\text{prior}$ are all constants and have bounded norms and $y$ as well as $R$ have all entries bounded by $1$. 
    Using the boundedness of $\|P_{\theta_s}\|_\text{Fr}$ on $[0,T]$, together equivalence of the Frobenius norm with the operator norm,  we now see that $m_{\theta_s}$ also stays bounded on $[0,T]$. Therefore, the empirical mean and covariance stay bounded, hence do all particles $\theta^j$.

    In Proposition \ref{prop:mfl_odes_existence} we have already seen that $m_s$ and $P_s$ stay bounded on $[0,T]$. By plugging the bounds into \eqref{equ:aldi mf particles} one can apply Grönwall there and obtain that the $\eta^j$ do not explode.
\end{proof}

\subsection{A priori bounds}\label{sec:apriori_moment_bound}

In this section we prove Proposition \ref{prop:apriori_moment_bound}:

\begin{proof}
  We will use that
  \[ \mathbb{E} \left[ \frac{1}{J} \sum | \theta^j_s - \eta^j_s |^2 \right]
     \leqslant 2 \left( \mathbb{E} \left[ \frac{1}{J} \sum | \theta^j_s |^2
     \right] +\mathbb{E} \left[ \frac{1}{J} \sum | \eta^j_s |^2 \right]
     \right) \]
  Since $P_s$ and $m_s$ are bounded on $[0, T]$ by Proposition
  \ref{prop:mfl_odes_existence} and $\eta^j_s$ are independent samples from
  $\mathcal{N} (m_0, P_0)$ we know that $\mathbb{E} \left[ \frac{1}{J} \sum |
  \eta^j_s |^2 \right] =\mathbb{E} [| \eta^0_s |^2]$ is bounded on $[0, T]$.
  What is left is to treat the first term.
  
  By \eqref{equ:aldiode} we know that
  \begin{align*}
    \frac{\mathd}{\mathd s} | \tilde{\theta}_s^j |^{2 p} & \leqslant  |
    \tilde{\theta}_s^j |^{2 (p - 1)} \left\langle \frac{\mathd}{\mathd s}
    \tilde{\theta}^j_s, \tilde{\theta}^j_s \right\rangle \leqslant |
    \tilde{\theta}_s^j |^{2 (p - 1)} (\| P_{\theta_s} \| + 1) |
    \tilde{\theta}^j_s |^2 \\
    & \leqslant | \tilde{\theta}^j_s |^{2 p} (\|P_{\theta_0} \|_{\tmop{op}} + 1),
  \end{align*}
  where we applied \eqref{equ:Ptheta oeprator norm groenwall}. We upper bound
  $\| P_{\theta_0} \|$ by $\tmop{tr} (P_{\theta_0})$, which is possible since
  $P_{\theta_0}$ is positive semidefinite. We use H\"older's inequality to see
  that
  \[ \mathbb{E} [| \tilde{\theta}_s^j |^{2 p}] \leqslant \mathbb{E} [|
     \tilde{\theta}_0^j |^{2 p} \exp (s \tmop{tr} (P_{\theta_0}))] \leqslant
     \mathbb{E} [| \tilde{\theta}_0^j |^{2 \tmop{pq}}]^{1 / q} \mathbb{E}
     [\exp (sq' \tmop{tr} (P_{\theta_0}))]^{1 / q'} . \]
  Without loss of generality, we can switch the coordinate system to one in
  which the coordinates of $\theta_0$ are independent, i.e. in which
  $P_{\tmop{prior}}$ is a diagonal matrix. We denote the largest eigenvalue
  of $P_{\tmop{prior}}$ by $v$. Then,
  \[ \tmop{tr} (\tilde{P}_{\theta_0}) = \frac{1}{J} \sum_{k = 1}^D \sum_{j =
     1}^J (\theta^j_{0, k} - m_{0, k})^2 \]
  is a sum of independent centred Gaussians. We get that
  \begin{eqnarray*}
  \begin{aligned}
    \mathbb{E} [\exp (sq' \tmop{tr} (P_{\theta_0}))] & \leqslant \prod_{k =
    1}^D \prod_{j = 1}^J \mathbb{E} \left[ \exp \left( \frac{sq'}{J}
    (\theta^j_{0, k} - m_{0, k})^2 \right) \right] \\ & \leqslant \prod_{k = 1}^D
    \prod_{j = 1}^J \mathbb{E} \left[ \exp \left( \frac{sq' \sqrt{v}}{J} Z^2
    \right) \right]
    \end{aligned}
  \end{eqnarray*}
  where $Z \sim \mathcal{N} (0, 1)$. We denote by $\alpha = \sqrt{v} sq'$ and
  see that
  \begin{equation*}
    \mathbb{E} [\exp (\alpha \tmop{tr} (P_{\theta_0}))]  \leqslant  \prod_{k
    = 1}^D \left( 1 - 2 \frac{\alpha}{J} \right)^{- J / 2}
  \end{equation*}
  However, the series $a_J = \left( 1 - 2 \frac{\alpha}{J} \right)^{- J / 2}$
  converges, since

  \begin{align*}
   \lim_{J \rightarrow \infty} \ln \left( \left( 1 - 2 \frac{\alpha}{J}
     \right)^{- J / 2} \right)& = \lim_{J \rightarrow \infty} \frac{\ln \left(
     1 - 2 \frac{\alpha}{J} \right)}{- \frac{2}{J}} = \lim_{J \rightarrow
     \infty} \frac{\left( \frac{2 \alpha / J^2}{1 - 2 \alpha / J}
     \right)}{\left( \frac{2}{J^2} \right)} \\ 
     &= \lim_{J \rightarrow \infty}
     \frac{\alpha}{1 - \frac{2 \alpha}{J}} = \alpha, 
    \end{align*}

  where we used L'Hopital's rule. Therefore, $\left( 1 - 2 \frac{\alpha}{J}
  \right)^{- J / 2} \rightarrow \exp (\alpha)$ and hence
  \[ \prod_{k = 1}^D \left( 1 - 2 \frac{\alpha}{J} \right)^{- J / 2}
     \rightarrow \exp (D \alpha) = \exp \left( D \sqrt{v} sq' \right) . \]
  Being a convergent series, it is in particular bounded. Therefore, for every
  fixed $s$ and $q'$, the quantity $\mathbb{E} [\exp (sq' \tmop{tr} (P_{\theta_0}))]$ will
  be bounded uniformly in $J$.
  
  \ 
\end{proof}

\subsection{Mean-field limit proof} \label{sec:mfl_deterministic_aldi_proof}
We now prove Theorem \ref{thm:mfl_deterministic_aldi}

\begin{proof}  
  We will make repeated use of the fact that all entries of $R$ and $y$ are
  positive and upper bounded by $1$. Therefore we can bound the norms of these
  terms by constants. Furthermore, we can also upper bound the norms of $m_0,
  P_0$ and $\Phi$ by constants. We will now fix
  some $T \in \mathbb{R}$ and prove the statement for any $s \leqslant T$.
  
  The sign $a \lesssim b$, means that $a \leqslant cb$, where $c$ is a
  constant that only depends on $m_0, P_0, d, \Phi$ and $T$.
  
  We define
  \begin{align*}
    \delta^j_s  =  \eta^j_s - \theta^j_s, \qquad
    \Delta_s  =  \left( \frac{1}{J} \sum | \delta_s^j |^2 \right)^{1 / 2} .
  \end{align*}
  where $\theta^j$ and $\eta^j$ are defined through \eqref{equ:aldiode} and
  \eqref{equ:aldi mf particles} respectively.
  Then,
  \begin{align*}
    \mathd \delta^j_s & = - \frac{1}{2} (P_s - P_{\theta_s}) \Big( \Phi R_s
    \Phi^{\rm T}(\eta^j_s - m_s) + \Phi (y_s - d))\\
    & \quad +   \frac{1}{2}
    P_{\tmop{prior}}^{- 1} (\eta^j_s + m_s - 2 m_{\tmop{prior}}) \Big) - \frac{1}{2} P_{\theta_s} \Big(\Phi (R_s - R_{\theta_s}) \Phi^{\rm T}(\eta^j_s -
    m_s)\\
    &  \quad + \Phi R_{\theta_s} \Phi^{\rm T}(\eta^j_s - \theta^j_s - m_s + m_{\theta_s})+
    \Phi (y_s - y_{\theta_s})\Big)\\
    &  \quad  - \frac{1}{2} P_{\theta_s} \Big(P_{\tmop{prior}}^{- 1} (\eta^j_s - \theta^j_s
    + m_s - m_{\theta_s})\Big) + \frac{1}{2} \Big(\eta^j_s - \theta^j_s - (m_s - m_{\theta_s})\Big)
   \end{align*}
  and therefore,
  \begin{align*}
      \frac{\mathd}{\mathd t} \Delta_s^2 
      & \lesssim  \frac{1}{J} 
      	\sum_{j = 1}^J 
      	| \delta^j_s |  \| P_s - P_{\theta_s} \|_{\tmop{op}} (| \eta_s^j -
      m_s | + 1 + | \eta_s^j - m_s + 2m_s - 2 m_{\tmop{prior}} |)\\
      & \quad +  \frac{1}{J} \sum_{j = 1}^J | \delta^j_s |  \| P_{\theta_s}
      \|_{\tmop{op}} (\| R_s - R_{\theta} \| | \eta^j_s - m_s | + | \delta_s^j |
      + | m_s - m_{\theta_s} | + | y_s - y_{\theta} |)\\
      & \quad+  \frac{1}{J} \sum_{j = 1}^J | \delta^j_s |  \| P_{\theta_s}
      \|_{\tmop{op}} (| \delta^j_s | + | m_s - m_{\theta} |)\\
      & \quad+  \frac{1}{J} \sum_{j = 1}^J | \delta_s^j |^2 + | \delta_s^j |  |
      m_s - m_{\theta} |\\
      & \quad \leqslant  \Delta_s \| P_s - P_{\theta_s} \|_{\tmop{op}}  (\tmop{tr}
      (P_{\eta_s})^{1 / 2} + 1)\\
      & \quad +  \Delta_s  \| P_{\theta_s} \|_{\tmop{op}} (\| R_s - R_{\theta} \|
      \tmop{tr} (P_{\eta_s})^{1 / 2} + \Delta_s + | m_s - m_{\theta_s} | + |
      y_s - y_{\theta} |)\\
      & \quad +  \Delta_s  \| P_{\theta_s} \|_{\tmop{op}} (\Delta_s + | m_s -
      m_{\theta} |)\\
      & \quad +  \Delta_s^2 + \Delta_s  | m_s - m_{\theta} |
  \end{align*}\label{equ:boundDs1}

  where we used H\"older's inequality and use that
  \begin{equation}
    | \eta_s^j - m_s + 2m_s - 2 m_{\tmop{prior}} | \leqslant | \eta_s^j - m_s |
    + 2 | m_s - m_{\tmop{prior}} | \lesssim | \eta_s^j - m_s | + 1
    \label{equ:double mean bound}
  \end{equation}
  for the first summand and
  \begin{equation*}
   \left( \frac{1}{J} \sum_{j = 1}^J | \eta_s^j - m_{\eta_s} |^2 \right)^{1
     / 2} = \tmop{tr} (P_{\eta_s})^{\frac{1}{2}} 
  \end{equation*}
  for the first and second summand. Furthermore, in \eqref{equ:double mean bound} we used that for
  a fixed $T$, the last term is bounded on $s \in [0, T]$, since $m_s$ is
  continuous (see Proposition \ref{prop:mfl_odes_existence}) and
  $m_{\tmop{prior}}$ is constant. For $\| P_s - P_{\theta_s} \|_{\tmop{op}}$,
  we get that
  \begin{align*}
    \| P_s - P_{\theta_s} \|_{\tmop{op}} & \leqslant  \| P_s - P_{\eta_s}
    \|_{\tmop{op}} + \| P_{\eta_s} - P_{\theta_s} \|_{\tmop{op}}\\
    & \leqslant  \| P_s - P_{\eta_s} \|_{\text{Fr}} + (\tmop{tr} (P_{\theta_s})^{1 /
    2} + \tmop{tr} (P_{\eta_s})^{1 / 2}) \Delta_s ,\\
    & \leqslant  H_{1 / 2} + (\tmop{tr} (P_{\theta_s})^{1 / 2} + \tmop{tr}
    (P_{\eta_s})^{1 / 2}) \Delta_s 
  \end{align*}
  where we used Lemma \ref{lemma:covariance difference bound} and Proposition \ref{prop:convergence_estimators}. We see that for a $L$-Lipschitz function $f$,
  \[ | f_{\eta_s} - f_{\theta_s} | \leqslant \frac{1}{J} \sum | f (\theta^j_s)
     - f (\eta^j_s) | \leqslant L \frac{1}{J} \sum | \theta^j_s - \eta_s^j |
     \leqslant L \Delta_s . \]
  We use that together with the fact that $R$ and $y$ are Lipschitz in all
  entries and Proposition \ref{prop:convergence_estimators} to obtain the
  following bounds:
  \begin{align*}
    | m_s - m_{\theta} | & \leqslant  | m_s - m_{\eta_s} | + | m_{\theta} -
    m_{\eta_s} | \leqslant H_{1 / 2} + \Delta_s, \\
    \| R_s - R_{\theta_s} \|_{\tmop{op}} & \leqslant  \| R_s - R_{\eta_s}
    \|_{\tmop{op}} + \| R_{\eta_s} - R_{\theta_s} \|_{\tmop{op}} \leqslant
    H_{1 / 2} + \Delta_s, \\
    | y_s - y_{\theta_s} | & \leqslant  | y_s - y_{\eta_s} | + | y_{\eta_s} -
    y_{\theta_s} | \leqslant H_{1 / 2} + \Delta_s,  
    \end{align*}
  where we used Proposition \ref{prop:convergence_estimators}. Therefore, we
  get that
 \begin{align*}
       \frac{\mathd}{\mathd t} \Delta_s^2 & \lesssim  \Delta_s (H_{1 / 2} +
       (\tmop{tr} (P_{\theta_s})^{1 / 2} + \tmop{tr} (P_{\eta_s})^{1 / 2})
       \Delta_s )  (\tmop{tr} (P_{\eta_s})^{1 / 2} + 1)\\
       & \quad +  \Delta_s  \| P_{\theta_s} \|_{\tmop{op}} (H_{1 / 2} \tmop{tr}
       (P_{\eta_s})^{1 / 2} + \Delta_s + H_{1 / 2})\\
       & \quad +  \Delta_s  \| P_{\theta_s} \|_{\tmop{op}} (\Delta_s + H_{1 /
       2})\\
       & \quad +  \Delta_s^2 + \Delta_s  (H_{1 / 2} + \Delta_s)
 \end{align*}
We now take the square root of $\Delta_s^2$ and apply Proposition
  \ref{prop:convergence_estimators} repeatedly to obtain
 \begin{align*}
       & \frac{\mathd}{\mathd t} \Delta_s \\
       \lesssim&\quad (H_{1 / 2} + (1 + H_{1 /
       4}) \Delta_s )  (1 + H_{1 / 4})
       + (1 + H_{1 / 2}) (H_{1 / 2} (1 + H_{1 / 4}) + \Delta_s + H_{1 /
       2})\\
       & \quad+  (1 + H_{1 / 2}) (\Delta_s + H_{1 / 2})
       +  \Delta_s  + (H_{1 / 2} + \Delta_s)
        \lesssim  \Delta_s + H_{1 / 4} \Delta_s + H_{1 / 2} .
\end{align*}
  We now assume that $\mathbb{E} [\Delta_{2, s}] \lesssim J^{- \alpha}$ and argue as in Equation \eqref{eq:eps_hoeldering}:
 \begin{align*}
    \frac{\mathd}{\mathd s} \mathbb{E} [\Delta_s] & \lesssim \mathbb{E}
    [\Delta_s] +\mathbb{E} [\Delta_s H_{1 / 4}] +\mathbb{E} [H_{1 / 2}]\\
    & \leqslant  \mathbb{E} [\Delta_s] +\mathbb{E} [\Delta_s^{\varepsilon}
    \Delta_s^{1 - \varepsilon} H_{1 / 4}] +\mathbb{E} [H_{1 / 2}]\\
    & \leqslant \mathbb{E} [\Delta_s] +\mathbb{E} [\Delta_s]^{1 -
    \varepsilon} \mathbb{E} [\Delta_s  (H_{1 / 4})^{1 /
    \varepsilon}]^{\varepsilon} + J^{- 1 / 2}\\
    & \leqslant  \mathbb{E} [\Delta_s] +\mathbb{E} [\Delta_s]^{1 -
    \varepsilon} \mathbb{E} [\Delta_s^2]^{\varepsilon / 2} \mathbb{E} [(H_{1 /
    4})^{2 / \varepsilon}]^{\varepsilon / 2} + J^{- 1 / 2}.
 \end{align*}
 From here on, we can finish the proof as in the proof sketch of Theorem \ref{thm:mfl_deterministic_aldi}.
  where we used H\"older with $p = \frac{1}{1 - \varepsilon}$ and $q =
  \frac{1}{\varepsilon}$ in the third inequality. By Proposition
  \ref{prop:apriori_moment_bound}, we can bound $\mathbb{E} [\Delta_s^2]^{1 /
  2}$ by $1$.
  Let us, for now, assume that
  \begin{equation}
    \mathbb{E} [\Delta_s] \leqslant J^{- \alpha} \label{equ:alpha decay
    assumption}
  \end{equation}
  for some $\alpha \geqslant 0$. Then, we now get that
 \begin{align*}
    \frac{\mathd}{\mathd s} \mathbb{E} [\Delta_s] & \lesssim  \mathbb{E}
    [\Delta_s] +\mathbb{E} [\Delta_s]^{1 - \varepsilon} J^{- 1 / 4} + J^{- 1 /
    2}\\
    & \leqslant  \mathbb{E} [\Delta_s] + J^{- \alpha (1 - \varepsilon) - 1 /
    4} + J^{- 1 / 2}
  \end{align*}
  In particular, we can apply Groenwall to obtain
  \[ \mathbb{E} [\Delta_s] \lesssim J^{- \beta} \]
  where $\beta = \min \left( \frac{1}{2}, \frac{1}{4} + \alpha (1 -
  \varepsilon) \right)$. Therefore, if the Assumption \eqref{equ:alpha decay
  assumption} is fulfilled for $\alpha = 0$, we obtain a rate of $J^{- 1 / 4}$.
  Then, by picking $\varepsilon$ small enough we can achieve any rate below $1
  / 2$. By improving the rate a third time, we can achieve $\alpha = 1/2$. 
  The only thing left to prove is that we can a priori choose $\alpha =
  0$, i.e. that $\mathbb{E} [\Delta_s]$ can be bounded independently of $J$.
  This is proven in Proposition \ref{prop:apriori_moment_bound}.
\end{proof}

\begin{proposition}
  \label{prop:convergence_estimators}Convergence of the estimators. Let
  $\eta^j_0 \sim \mathcal{N} (m_0, P_0)$ i.i.d. and denote by $m_{\eta_0}$ and
  $P_{\eta_0}$ the empirical mean and covariance of the ensemble. Then, for
  each $p$, there exist constants $C_p$ such that
  \begin{enumerate}
    \item $\| m_{\eta_s} - m_s \|_p \leqslant C_{p, T} J^{- 1 / 2}$.
    
    \item $\| \tmop{tr} (P_{\eta_s})^{1 / 2} - \tmop{tr} (P_s)^{1 / 2} \|_p
    \leqslant C_{p, T} J^{- 1 / 4}$.
    
    \item $\| \| P_{\eta_s} - P_s \|_{\tmop{op}} \|_p \leqslant C_{p, T} J^{-
    1 / 2}$.
    
    \item $\| y_{\eta_s} - y_s \|_p \leqslant C_{p, T} J^{- 1 / 2}$.
    
    \item $\| \| R_{\eta_s} - R_s \|_{\tmop{op}} \|_p \leqslant C_{p, T} J^{-
    1 / 2}$.
    
    \item $\| \tmop{tr} (P_{\theta_s})^{1 / 2} \|_p \lesssim C_{p, T} (1 +
    J^{- 1 / 4})$.
    
    \item $\| P_{\theta_s} \|_{\tmop{op}} \lesssim C_{p, T} (1 + J^{- 1 / 2})$.
  \end{enumerate}
\end{proposition}

\begin{proof}
  \tmtextbf{Item 1:} We use Proposition \ref{prop:mfl_odes_existence} to see
  that there is a unique solution $(m_s, P_s)_{s \in [0, T]}$ for $m$ and $P$
  on $[0, T]$. From this, we can deduce that all moments of $\eta_s^j$ on $[0,
  T]$ can be bounded uniformly. We will use this, since the $\xi_i$ in Lemma
  \ref{lemma:high_moments_bound} will mostly depend on $\eta_s^j$, and
  therefore the bound, depending only on $p$ and $\| \xi_i \|_{2 p}$ can be
  made uniform in $[0, T]$.
  
  We get that
  \[ m_{\eta_s} - m_s = \frac{1}{J} \sum_{j = 1}^J (\eta^j_s - m_s), \]
  to which we can apply Lemma \ref{lemma:high_moments_bound} uniformly over
  $[0, T]$ to prove $1.$
  \\
  \tmtextbf{Item 2:} We define
  \begin{equation*}
    \tilde{P}_{\eta_s}  = \frac{1}{J} \sum_{j = 1}^J (\eta_s^j - m_s)
    (\eta^j_s - m_s)^{\rm T} = P_{\eta_s} + (m_{\eta_s} - m_s) (m_{\eta_s} - m_s)^{\rm T}.
  \end{equation*}
  We can write
  \[ \tmop{tr} (\tilde{P}_{\eta_s}) - \tmop{tr} (P_s) = \frac{1}{J} \sum_{j =
     1}^J \xi^j_s \]
  with $\xi^j_s = (| \eta^j_s - m_s |^2 -\mathbb{E} [| \eta^j_s - m_s |^2])$
  being independent, centred random variables. Furthermore, all their moments
  can be bounded uniformly on $[0, T]$ again, since the moments of $\eta_s^j$
  are bounded uniformly on $[0, T]$. We can then apply Lemma
  \ref{lemma:high_moments_bound} to conclude that
  \[ \| \tmop{tr} (\tilde{P}_{\eta_s}) - \tmop{tr} (P_s) \|_p \leqslant C_p
     J^{- 1 / 2} . \]
  Now we use that $| a^{1 / 2} - b^{1 / 2} | \leqslant | a - b |^{1 / 2}$ and
  Jensens inequality to see that
  \[ \| \tmop{tr} (\tilde{P}_{\eta_s})^{1 / 2} - \tmop{tr} (P_s)^{1 / 2} \|_p
     \leqslant C_p J^{- 1 / 4} \]
  Furthermore, we get that
\begin{align*}
      \tmop{tr} (\tilde{P}_{\eta_s}) 
	   = \tmop{tr} (P_{\eta_s}) + (m_{\eta_s} - m_s)^2
\end{align*}
  Therefore, by also using item $1$ again, we get
  \begin{equation*}
    \| \tmop{tr} (P_{\eta_s})^{1 / 2} - \tmop{tr} (P_s)^{1 / 2} \|_p = \|
    m_{\eta_s}  - m_s \|_p + \| \tmop{tr} (\tilde{P}_{\eta_s})^{1 / 2} -
    \tmop{tr} (P_s)^{1 / 2} \|_p \leqslant C_p  J^{- 1 / 4} .
  \end{equation*}
  \tmtextbf{Item 3:} We use that
  \begin{align*}
    \| \tilde{P}_{\eta_s} - P_s \|_{\text{Fr}}^2 & \lesssim  \sum_{n, m = 1}^D \left(
    \frac{1}{J} \sum_{j = 1}^J (\eta_{s, n}^j \eta_{s, m}^j - m_s -\mathbb{E}
    [\eta_{s, n}^j \eta_{s, m}^j - m_s]) \right)^2,
  \end{align*}
  and apply Lemma \ref{lemma:high_moments_bound} to all $D^2$ terms (in $m, n$)
  on the right-hand side, uniformly on $[0, T]$. Since we are taking their
  square, they will behave like $J^{- 1}$. Therefore, by using that the
  operator norm is equivalent to the Frobenius norm, we obtain
  \begin{align*}
       \| \| \tilde{P}_{\eta_s} - P_s \|_{\tmop{op}}^2 \|_p & \leqslant  C_p 
       J^{- 1} .
 \end{align*}
  Furthermore, $\| P_{\eta_s} - \tilde{P}_{\eta_s} \|_{\tmop{op}} = |
  m_{\eta_s} - m_s |^2$. Therefore,
  \begin{align*}
    \| \| P_{\eta_s} - P_s \|_{\tmop{op}}  \|_p & \leqslant  \| \|
    \tilde{P}_{\eta_s} - P_s \|_{\tmop{op}}  \|_p + \| | m_{\eta_s} - m_s |^2
    \|_p \lesssim J^{- 1 / 2} + J^{- 1},
  \end{align*}
  where we used Item 1. This yields the desired result.
  \\
  \tmtextbf{Item 4\&5:} Now let $f : \mathbb{R}^D \rightarrow \mathbb{R}$ be a
  $f_L$-Lipschitz function. Then
  \[ \mathbb{E} [| f (\eta^j_s) -\mathbb{E}_{\eta_s} [f (\eta_s)] |^p]^{1 / p}
     =\mathbb{E} [| \mathbb{E}_{\eta_s} [f (\eta^j_s) - f (\eta_s)] |^p]^{1 /
     p} \leqslant f_L \mathbb{E} [| \eta_s^j - \eta_s |^p]^{1 / p}, \]
  and therefore the moments of $f (\eta^j_s)$ can also be bounded uniformly on
  $s \in [0, T]$ and we can apply Lemma \ref{lemma:high_moments_bound}. Since
  the entries of $y$ are all Lipschitz, we can apply this element-wise and get
  item 4. The same holds for all the entries of $R$. Applying it to each entry
  and using that the operator norm of a positive diagonal matrix is bounded by
  its highest entry, we get item $5$.
  \\
  \tmtextbf{Item 6:} We get that
  \begin{align*}
    \frac{\mathd}{\mathd s} \tmop{tr} (P_{\theta_s}) & =  \tmop{tr} \left( -
    \frac{1}{2} P_{\theta_s} \Phi R_{\theta_s} \Phi^{\rm T}P_{\theta_s} - \frac{1}{2}
    (P_{\tmop{prior}}^{- 1} P_{\theta_s} + P_{\theta_s} P_{\tmop{prior}}^{- 1}) +
    \frac{1}{2} P_{\theta_s} \right)\\
    & \lesssim  \| P_{\tmop{prior}}^{- 1} \| _{\text{Fr}} \| P_{\theta_s} \|_{\text{Fr}} + \tmop{tr} (P_{\theta_s})\\
    & \lesssim  \| P_{\tmop{prior}}^{- 1} \| _{\text{Fr}} \tmop{tr} (P_{\theta_s}) +
    \tmop{tr} (P_{\theta_s})
  \end{align*}
  where we denote by $\| P  \|_{\text{Fr}}$ the Frobenius-norm, which is upper bounded
  by a multiple of the operator norm, which itself is upper-bounded by the
  trace (for positive semidefinite matrices). Therefore, we can now apply
  Groenwall and bound
  \begin{align*}
    \tmop{tr} (P_{\theta_s})^{1 / 2} & \lesssim  \tmop{tr} (P_{\theta_0})^{1
    / 2} \leqslant | \tmop{tr} (P_{\theta_0})^{1 / 2} - \tmop{tr} (P_0)^{1 /
    2} | + \tmop{tr} (P_0)^{1 / 2}
  \end{align*}
  Since $\tmop{tr} (P_{\theta_0}) = \tmop{tr} (P_{\eta_0})$ we can now apply
  Item 2 to the first term. The last term is constant. The claim then follows.
  \\
  \tmtextbf{Item 7:} For the Frobenius-norm we get that
  \begin{align*}
    \frac{\mathd}{\mathd s} \| P_{\theta_s} \|_{\text{Fr}}^2 & =  \tmop{tr} \left( -
    \frac{1}{2} P_{\theta_s} \Phi R_{\theta} \Phi^{\rm T}P_{\theta_s}^2 -
    \frac{1}{2} (P_{\tmop{prior}}^{- 1} P_{\theta_s}^2 + P_{\theta_s}^2
    P_{\tmop{prior}}^{- 1}) + \frac{1}{2} P_{\theta_s}^2 \right)\\
    & =  - \frac{1}{2} \tmop{tr} (P_{\theta_s}^{3 / 2} \Phi R_{\theta}
    \Phi^{\rm T}P_{\theta_s}^{3 / 2} + P_{\theta_s} P_{\tmop{prior}}^{- 1}
    P_{\theta_s}) + \frac{1}{2} \tmop{tr} (P_{\theta_s}^2)\\
    & \leqslant  \frac{1}{2} \| P_{\theta_s} \| .
  \end{align*}
  We apply Groenwall and using that the operator norm and the Frobenius norm
  are equivalent on $\mathbb{R}^D$, we get that
  \begin{align}
    \| P_{\theta_s} \|_{\tmop{op}} & \lesssim  \| P_{\theta_0} \|_{\tmop{op}}
    \leqslant \| P_{\theta_0} - P_0 \|_{\tmop{op}} + \| P_0 \|_{\tmop{op}} . 
    \label{equ:Ptheta oeprator norm groenwall}
  \end{align}
  We now apply Item 3 to the first term on the right-hand side and use that
  the last term is constant. The claim follows.
\end{proof}

\printbibliography

@article{pidstrigach2022affine,
	title="Affine-invariant ensemble transform methods for logistic regression",
	author="Pidstrigach, Jakiw and Reich, Sebastian",
	journal="Foundations of Computational Mathematics",
	pages="1--34",
	year="2022",
	publisher="Springer"
}

@article{garbuno2020interacting,
	title="Interacting Langevin diffusions: Gradient structure and ensemble Kalman sampler",
	author="Garbuno-Inigo, Alfredo and Hoffmann, Franca and Li, Wuchen and Stuart, Andrew M",
	journal="SIAM Journal on Applied Dynamical Systems",
	volume="19",
	number="1",
	pages="412--441",
	year="2020",
	publisher="SIAM"
}

@article{ding2021ensemble,
	title="Ensemble Kalman inversion: mean-field limit and convergence analysis",
	author="Ding, Zhiyan and Li, Qin",
	journal="Statistics and Computing",
	volume="31",
	number="1",
	pages="1--21",
	year="2021",
	publisher="Springer"
}

@article{ding2021ensemblesampler,
    title = "Ensemble Kalman Sampler: Mean-field Limit and Convergence Analysis",
    author = "Ding, Zhiyan and Li, Qin",
    journal = "SIAM Journal on Mathematical Analysis",
    volume = "53",
    number = "2",
    pages = "1546-1578",
    year = "2021",
    doi = "10.1137/20M1339507",
    URL = "https://doi.org/10.1137/20M1339507"
}

@article{schillings2017analysis,
	title="Analysis of the ensemble Kalman filter for inverse problems",
	author="Schillings, Claudia and Stuart, Andrew M",
	journal="SIAM Journal on Numerical Analysis",
	volume="55",
	number="3",
	pages="1264--1290",
	year="2017",
	publisher="SIAM"
}

@article{schillings2018convergence,
	title="Convergence analysis of ensemble Kalman inversion: the linear, noisy case",
	author="Schillings, Claudia and Stuart, Andrew M",
	journal="Applicable Analysis",
	volume="97",
	number="1",
	pages="107--123",
	year="2018",
	publisher="Taylor \& Francis"
}

@article{dieci1999smooth,
	title="On smooth decompositions of matrices",
	author="Dieci, Luca and Eirola, Timo",
	journal="SIAM Journal on Matrix Analysis and Applications",
	volume="20",
	number="3",
	pages="800--819",
	year="1999",
	publisher="SIAM"
}

@article{Huang_2022,
doi = {10.1088/1361-6420/ac99fa},
url = {https://dx.doi.org/10.1088/1361-6420/ac99fa},
year = {2022},
month = oct,
publisher = {IOP Publishing},
volume = {38},
number = {12},
pages = {125006},
author = {Daniel Zhengyu Huang and Jiaoyang Huang and Sebastian Reich and Andrew M Stuart},
title = {Efficient derivative-free Bayesian inference for large-scale inverse problems},
journal = {Inverse Problems}
}

@inproceedings{Kristiadi2020BeingBE,
  title={Being Bayesian, Even Just a Bit, Fixes Overconfidence in ReLU Networks},
  author={Agustinus Kristiadi and Matthias Hein and Philipp Hennig},
  booktitle={International Conference on Machine Learning},
  year={2020}
}

@InProceedings{Hein_2019_CVPR,
author = {Hein, Matthias and Andriushchenko, Maksym and Bitterwolf, Julian},
title = {Why ReLU Networks Yield High-Confidence Predictions Far Away From the Training Data and How to Mitigate the Problem},
booktitle = {Proceedings of the IEEE/CVF Conference on Computer Vision and Pattern Recognition (CVPR)},
year = {2019}
}

@InProceedings{pmlr-v70-guo17a,
  title = 	 {On Calibration of Modern Neural Networks},
  author =       {Chuan Guo and Geoff Pleiss and Yu Sun and Kilian Q. Weinberger},
  booktitle = 	 {Proceedings of the 34th International Conference on Machine Learning},
  pages = 	 {1321--1330},
  year = 	 {2017},
  editor = 	 {Precup, Doina and Teh, Yee Whye},
  volume = 	 {70},
  series = 	 {Proceedings of Machine Learning Research},
  month = 	 Aug,
  publisher =    {PMLR},
  url = 	 {https://proceedings.mlr.press/v70/guo17a.html}
}

@misc{sharma2023bayesian,
      title={Do Bayesian Neural Networks Need To Be Fully Stochastic?}, 
      author={Mrinank Sharma and Sebastian Farquhar and Eric Nalisnick and Tom Rainforth},
      year={2023},
      eprint={2211.06291},
      archivePrefix={arXiv},
      primaryClass={cs.LG}
}

@PhdThesis{Gal2016UncertaintyID,
  title={Uncertainty in Deep Learning},
  author={Gal, Yarin},
  year={2016},
  school={University of Cambridge}
}

@inproceedings{10.5555/3045118.3045290,
author = {Blundell, Charles and Cornebise, Julien and Kavukcuoglu, Koray and Wierstra, Daan},
title = {Weight Uncertainty in Neural Networks},
year = {2015},
publisher = {JMLR.org},
booktitle = {Proceedings of the 32nd International Conference on International Conference on Machine Learning - Volume 37},
pages = {1613–1622},
numpages = {10},
location = {Lille, France},
series = {ICML'15}
}

@InProceedings{pmlr-v80-zhang18l,
  title = 	 {Noisy Natural Gradient as Variational Inference},
  author =       {Zhang, Guodong and Sun, Shengyang and Duvenaud, David and Grosse, Roger},
  booktitle = 	 {Proceedings of the 35th International Conference on Machine Learning},
  pages = 	 {5852--5861},
  year = 	 {2018},
  editor = 	 {Dy, Jennifer and Krause, Andreas},
  volume = 	 {80},
  series = 	 {Proceedings of Machine Learning Research},
  publisher =    {PMLR},
  url = 	 {https://proceedings.mlr.press/v80/zhang18l.html}
}

@inproceedings{Lakshminarayanan2016SimpleAS,
  title={Simple and Scalable Predictive Uncertainty Estimation using Deep Ensembles},
  author={Balaji Lakshminarayanan and Alexander Pritzel and Charles Blundell},
  booktitle={NIPS},
  year={2016}
}

@inproceedings{Daxbergeretal21,
  title = {Bayesian Deep Learning via Subnetwork Inference},
  author = {Daxberger, E. and Nalisnick, E. and Allingham, J. and Antorán, J. and Hernández-Lobato, J. M.},
  booktitle = {Proceedings of 38th International Conference on Machine Learning (ICML)},
  volume = {139},
  pages = {2510--2521},
  series = {Proceedings of Machine Learning Research},
  editors = {Meila, Marina and Zhang, Tong},
  publisher = {PMLR},
  month = jul,
  year = {2021},
  doi = {},
  month_numeric = {7},
  url = {https://proceedings.mlr.press/v139/daxberger21a.html}
}

@InProceedings{pmlr-v37-snoek15,
  title = 	 {Scalable Bayesian Optimization Using Deep Neural Networks},
  author = 	 {Snoek, Jasper and Rippel, Oren and Swersky, Kevin and Kiros, Ryan and Satish, Nadathur and Sundaram, Narayanan and Patwary, Mostofa and Prabhat, Mr and Adams, Ryan},
  booktitle = 	 {Proceedings of the 32nd International Conference on Machine Learning},
  pages = 	 {2171--2180},
  year = 	 {2015},
  editor = 	 {Bach, Francis and Blei, David},
  volume = 	 {37},
  series = 	 {Proceedings of Machine Learning Research},
  address = 	 {Lille, France},
  month = 	 Jul,
  publisher =    {PMLR},
  url = 	 {https://proceedings.mlr.press/v37/snoek15.html}
}

@inproceedings{
liang2018enhancing,
title={Enhancing The Reliability of Out-of-distribution Image Detection in Neural Networks},
author={Shiyu Liang and Yixuan Li and R. Srikant},
booktitle={International Conference on Learning Representations},
year={2018},
url={https://openreview.net/forum?id=H1VGkIxRZ},
}

@book{10.5555/1206873, 
author = {Evensen, Geir}, 
title = {Data Assimilation: The Ensemble Kalman Filter}, year = {2006},
isbn = {354038300X}, 
publisher = {Springer-Verlag},
address = {Berlin, Heidelberg} }

@article{Reich2011ADS,
  title={A dynamical systems framework for intermittent data assimilation},
  author={Sebastian Reich},
  journal={BIT Numerical Mathematics},
  year={2011},
  volume={51},
  pages={235-249},
  doi = {10.1007/s10543-010-0302-4}
}

@article{articleReich,
author = {Cotter, C.J. and Reich, Sebastian},
year = {2013},
pages = {91-134},
title = {Ensemble filter techniques for intermittent data assimilation},
volume = {13},
journal = {Radon Ser. Comput. Appl. Math.}
}

@article{Kovachki_2019,
doi = {10.1088/1361-6420/ab1c3a},
url = {https://dx.doi.org/10.1088/1361-6420/ab1c3a},
year = {2019},
month = aug,
publisher = {IOP Publishing},
volume = {35},
number = {9},
pages = {095005},
author = {Nikola B Kovachki and Andrew M Stuart},
title = {Ensemble Kalman inversion: a derivative-free technique for machine learning tasks},
journal = {Inverse Problems}
}

@InProceedings{pmlr-v48-gal16,
  title = 	 {Dropout as a Bayesian Approximation: Representing Model Uncertainty in Deep Learning},
  author = 	 {Gal, Yarin and Ghahramani, Zoubin},
  booktitle = 	 {Proceedings of The 33rd International Conference on Machine Learning},
  pages = 	 {1050--1059},
  year = 	 {2016},
  editor = 	 {Balcan, Maria Florina and Weinberger, Kilian Q.},
  volume = 	 {48},
  series = 	 {Proceedings of Machine Learning Research},
  address = 	 {New York, New York, USA},
  publisher =    {PMLR},
  url = 	 {https://proceedings.mlr.press/v48/gal16.html}
}

@inproceedings{NEURIPS2019_8558cb40,
 author = {Ovadia, Yaniv and Fertig, Emily and Ren, Jie and Nado, Zachary and Sculley, D. and Nowozin, Sebastian and Dillon, Joshua and Lakshminarayanan, Balaji and Snoek, Jasper},
 booktitle = {Advances in Neural Information Processing Systems},
 editor = {H. Wallach and H. Larochelle and A. Beygelzimer and F. d\textquotesingle Alch\'{e}-Buc and E. Fox and R. Garnett},
 pages = {},
 publisher = {Curran Associates, Inc.},
 title = {Can you trust your model\textquotesingle s uncertainty?  Evaluating predictive uncertainty under dataset shift},
 url = {https://proceedings.neurips.cc/paper_files/paper/2019/file/8558cb408c1d76621371888657d2eb1d-Paper.pdf},
 volume = {32},
 year = {2019}
}

@misc{Wilson2020TheCF,
      title={The Case for Bayesian Deep Learning}, 
      author={Andrew Gordon Wilson},
      year={2020},
      eprint={2001.10995},
      archivePrefix={arXiv},
      primaryClass={cs.LG}
}

@article{ABDAR2021243,
title = {A review of uncertainty quantification in deep learning: Techniques, applications and challenges},
journal = {Information Fusion},
volume = {76},
pages = {243-297},
year = {2021},
issn = {1566-2535},
doi = {https://doi.org/10.1016/j.inffus.2021.05.008},
url = {https://www.sciencedirect.com/science/article/pii/S1566253521001081},
author = {Moloud Abdar and Farhad Pourpanah and Sadiq Hussain and Dana Rezazadegan and Li Liu and Mohammad Ghavamzadeh and Paul Fieguth and Xiaochun Cao and Abbas Khosravi and U. Rajendra Acharya and Vladimir Makarenkov and Saeid Nahavandi}
}

@article{MacKay1992TheEF,
  title={The Evidence Framework Applied to Classification Networks},
  author={David John Cameron MacKay},
  journal={Neural Computation},
  year={1992},
  volume={4},
  pages={720-736}
}

@article{MacKay1991APB,
  title={A Practical Bayesian Framework for Backprop Networks},
  author={David John Cameron MacKay},
  journal={Neural Computation},
  year={1991}
}

@article{10.1162/neco.1992.4.3.448,
    author = {MacKay, David J. C.},
    title = "{A Practical Bayesian Framework for Backpropagation Networks}",
    journal = {Neural Computation},
    volume = {4},
    number = {3},
    pages = {448-472},
    year = {1992},
    month = 05,
    issn = {0899-7667},
    doi = {10.1162/neco.1992.4.3.448},
    url = {https://doi.org/10.1162/neco.1992.4.3.448},
    eprint = {https://direct.mit.edu/neco/article-pdf/4/3/448/812348/neco.1992.4.3.448.pdf},
}

@inproceedings{NIPS2011_7eb3c8be,
 author = {Graves, Alex},
 booktitle = {Advances in Neural Information Processing Systems},
 editor = {J. Shawe-Taylor and R. Zemel and P. Bartlett and F. Pereira and K.Q. Weinberger},
 pages = {},
 publisher = {Curran Associates, Inc.},
 title = {Practical Variational Inference for Neural Networks},
 url = {https://proceedings.neurips.cc/paper_files/paper/2011/file/7eb3c8be3d411e8ebfab08eba5f49632-Paper.pdf},
 volume = {24},
 year = {2011}
}

@inproceedings{Welling2011BayesianLV,
  title={Bayesian Learning via Stochastic Gradient Langevin Dynamics},
  author={Max Welling and Yee Whye Teh},
  booktitle={International Conference on Machine Learning},
  year={2011}
}

@book{neal2012bayesian,
  author = {Neal, Radford M},
  publisher = {Springer Science \& Business Media},
  title = {Bayesian learning for neural networks},
  volume = 118,
  year = 2012
}

@article{neal2011mcmc,
  title={MCMC using Hamiltonian dynamics},
  author={Neal, Radford M and others},
  journal={Handbook of markov chain monte carlo},
  volume={2},
  number={11},
  pages={2},
  year={2011},
  publisher={Chapman and Hall/CRC}
}

@misc{Haber2018NeverLB,
      title={Never look back - A modified EnKF method and its application to the training of neural networks without back propagation}, 
      author={Eldad Haber and Felix Lucka and Lars Ruthotto},
      year={2018},
      eprint={1805.08034},
      archivePrefix={arXiv},
      primaryClass={math.NA}
}

@article{doi:10.1137/19M1304891,
author = {Garbuno-Inigo, Alfredo and N\"{u}sken, Nikolas and Reich, Sebastian},
title = {Affine Invariant Interacting Langevin Dynamics for Bayesian Inference},
journal = {SIAM Journal on Applied Dynamical Systems},
volume = {19},
number = {3},
pages = {1633-1658},
year = {2020},
doi = {10.1137/19M1304891},

URL = {   
        https://doi.org/10.1137/19M1304891

},
eprint = {  
        https://doi.org/10.1137/19M1304891
}
}

@article{https://doi.org/10.1029/94JC00572,
author = {Evensen, Geir},
title = {Sequential data assimilation with a nonlinear quasi-geostrophic model using Monte Carlo methods to forecast error statistics},
journal = {Journal of Geophysical Research: Oceans},
volume = {99},
number = {C5},
pages = {10143-10162},
doi = {https://doi.org/10.1029/94JC00572},
url = {https://agupubs.onlinelibrary.wiley.com/doi/abs/10.1029/94JC00572},
year = {1994}
}

@article{doi:10.1137/19M1303162,
author = {Reich, Sebastian and Weissmann, Simon},
title = {Fokker--Planck Particle Systems for Bayesian Inference: Computational Approaches},
journal = {SIAM/ASA Journal on Uncertainty Quantification},
volume = {9},
number = {2},
pages = {446-482},
year = {2021},
doi = {10.1137/19M1303162},

URL = { 
    
        https://doi.org/10.1137/19M1303162
    
    

},
eprint = { 
    
        https://doi.org/10.1137/19M1303162    

}
}

@Article{e23080990,
AUTHOR = {Galy-Fajou, Théo and Perrone, Valerio and Opper, Manfred},
TITLE = {Flexible and Efficient Inference with Particles for the Variational Gaussian Approximation},
JOURNAL = {Entropy},
VOLUME = {23},
YEAR = {2021},
NUMBER = {8},
ARTICLE-NUMBER = {990},
URL = {https://www.mdpi.com/1099-4300/23/8/990},
PubMedID = {34441130},
ISSN = {1099-4300},
DOI = {10.3390/e23080990}
}

@book{reich_cotter_2015, place={Cambridge}, title={Probabilistic Forecasting and Bayesian Data Assimilation}, DOI={10.1017/CBO9781107706804}, publisher={Cambridge University Press}, author={Reich, Sebastian and Cotter, Colin}, year={2015}}

@inproceedings{10.1117/12.839590,
author = {Fred Daum and Jim Huang and Arjang Noushin},
title = {{Exact particle flow for nonlinear filters}},
volume = {7697},
booktitle = {Signal Processing, Sensor Fusion, and Target Recognition XIX},
editor = {Ivan Kadar},
organization = {International Society for Optics and Photonics},
publisher = {SPIE},
pages = {769704},
year = {2010},
doi = {10.1117/12.839590},
URL = {https://doi.org/10.1117/12.839590}
}

@article{fournier2015rate,
  title={On the rate of convergence in Wasserstein distance of the empirical measure},
  author={Fournier, Nicolas and Guillin, Arnaud},
  journal={Probability theory and related fields},
  volume={162},
  number={3-4},
  pages={707--738},
  year={2015},
  publisher={Springer}
}

@article{galy2021flexible,
  title={Flexible and efficient inference with particles for the variational Gaussian approximation},
  author={Galy-Fajou, Th{\'e}o and Perrone, Valerio and Opper, Manfred},
  journal={Entropy},
  volume={23},
  number={8},
  pages={990},
  year={2021},
  publisher={MDPI}
}

@article{https://doi.org/10.1002/qj.2186,
author = {Amezcua, Javier and Ide, Kayo and Kalnay, Eugenia and Reich, Sebastian},
title = {Ensemble transform Kalman–Bucy filters},
journal = {Quarterly Journal of the Royal Meteorological Society},
volume = {140},
number = {680},
pages = {995-1004},
doi = {https://doi.org/10.1002/qj.2186},
url = {https://rmets.onlinelibrary.wiley.com/doi/abs/10.1002/qj.2186},
year = {2014}
}

@misc{chen2023sampling,
      title={Sampling via Gradient Flows in the Space of Probability Measures}, 
      author={Yifan Chen and Daniel Zhengyu Huang and Jiaoyang Huang and Sebastian Reich and Andrew M Stuart},
      year={2023},
      eprint={2310.03597},
      archivePrefix={arXiv},
      primaryClass={stat.ML}
}

\end{document}